\documentclass[nohyperref]{article}

\usepackage{microtype}
\usepackage{graphicx}
\usepackage{booktabs} 
\usepackage{multirow}
\usepackage{amsmath,amsfonts,amssymb}
\usepackage{mathtools}
\usepackage{dsfont}

\usepackage{subcaption}
\usepackage{caption}

\usepackage{hyperref}

\usepackage[accepted]{icml2022}

\usepackage{amsmath}
\usepackage{amssymb}
\usepackage{mathtools}

\usepackage[capitalize,noabbrev]{cleveref}

\usepackage{amsthm}
    {
        \theoremstyle{plain}
        \newtheorem{assumption}{Assumption}
        \newtheorem{thm}{Theorem}[section]
        \newtheorem{lem}[thm]{Lemma}

        \newtheorem{remark}{Remark}
        
        \theoremstyle{plain}
    }
  
\makeatletter
\newtheorem*{rep@prop}{\rep@title}
\newcommand{\newrepprop}[2]{%
\newenvironment{rep#1}[1]{%
    \def\rep@title{#2 \ref{##1}}%
        \begin{rep@prop}
    }%
    {\end{rep@prop}}
}

\newtheorem*{rep@thm}{\rep@title}
\newcommand{\newrepthm}[2]{%
\newenvironment{rep#1}[1]{%
    \def\rep@title{#2 \ref{##1}}%
        \begin{rep@thm}
    }%
    {\end{rep@thm}}
}

\newrepprop{prop}{Proposition}
\newrepthm{thm}{Theorem}

\theoremstyle{definition}

\newcommand{\FedAvg}{\texttt{FedAvg}}

\newcommand{\euler}{e}
\newcommand{\knnper}{\texttt{kNN-Per}}
\newcommand{\pFedGP}{\texttt{pFedGP}}

\newcommand{\sS}{\mathcal{S}}
\newcommand{\Ditto}{\texttt{Ditto}}

\renewcommand{\vec}[1]{\mathbf{#1}}                                                                     
\DeclareMathOperator{\disc}{disc}

\DeclareMathOperator*{\E}{\mathbb{E}}

\DeclarePairedDelimiter\ceil{\lceil}{\rceil}

\DeclareMathOperator*{\argmin}{arg\,min}

\usepackage[textsize=tiny]{todonotes}

\icmltitlerunning{Personalized Federated Learning through Local Memorization}

\begin{document}

\twocolumn[
\icmltitle{Personalized Federated Learning through Local Memorization}



\icmlsetsymbol{equal}{*}

\begin{icmlauthorlist}
\icmlauthor{Othmane Marfoq}{inria,accenture}
\icmlauthor{Giovanni Neglia}{inria}
\icmlauthor{Laetitia Kameni}{accenture}
\icmlauthor{Richard Vidal}{accenture}

\end{icmlauthorlist}

\icmlaffiliation{inria}{Inria, Université Côte d’Azur, Sophia Antipolis, France}
\icmlaffiliation{accenture}{Accenture Labs, Sophia Antipolis, France}

\icmlcorrespondingauthor{Othmane Marfoq}{othmane.marfoq@inria.fr}

\icmlkeywords{Machine Learning, Federated Learning}

\vskip 0.3in
]



\printAffiliationsAndNotice{}  

\begin{abstract}
    Federated learning allows clients to collaboratively learn statistical models while keeping their data local. Federated learning was originally used to train a unique global model to be served to all clients, but this approach might be sub-optimal when clients' local data distributions are heterogeneous. In order to tackle this limitation, recent \emph{personalized federated learning} methods train a separate model for each client while still leveraging the knowledge available at other clients. In this work, we exploit the ability of deep neural networks to extract high quality vectorial representations (embeddings) from non-tabular data, e.g., images and text, to propose a personalization mechanism based on local memorization. Personalization is obtained by interpolating a collectively trained global model with a local $k$-nearest neighbors (kNN) model based on the shared representation provided by the global model. We provide generalization bounds for the proposed approach in the case of binary classification, and we show on a suite of federated datasets that this approach achieves significantly higher accuracy and fairness than state-of-the-art methods.
    
\end{abstract}

\section{Introduction}
\label{sec:introduction}
Heterogeneity is a core and fundamental challenge in Federated Learning (FL)~\citep{li2020federated, kairouz2019advances}. Indeed, clients highly differ both in size and distribution of their local datasets (\emph{statistical heterogeneity}), and in their storage and computational capabilities (\emph{system heterogeneity}).
Those two aspects challenge the assumption that clients should train a common global model, as pursued in many federated learning papers \citep{mcmahan2017communication, konevcny2016federated, Sahu2018OnTC, karimireddy2020scaffold, mohri2019agnostic}. 
In fact, all clients should be content with a model's architecture constrained by the minimum common capabilities.
Even when clients have similar hardware (e.g., they are all smartphones), in presence of statistical heterogeneity, a global model may be arbitrarily bad for some clients raising important fairness concerns~\citep{li2021ditto}.

Motivated by  the recent success of memorization techniques based on nearest neighbours for natural language processing, ~\citep{khandelwal2019generalization, khandelwal2020nearest}, computer vision~\citep{papernot2018deep, orhan2018simple}, and few-shot classification~\cite{snell2017prototypical, wang2019simpleshot}, we propose \knnper, a personalized FL algorithm based on local memorization. \knnper{} combines a global model trained collectively (e.g., via \FedAvg~\citep{mcmahan2017communication}) with a kNN model on a client's local datastore.
The global model also  provides the shared representation used by the local kNN.
Local memorization at each FL client can  capture the client's local distribution shift with respect to the global distribution. Indeed, our experiments show that  memorization is more beneficial when the distribution shift is larger. The generalization bound in Sec.~\ref{sec:method} contributes to justify our empirical findings as well as those in \citep{khandelwal2019generalization,khandelwal2020nearest, orhan2018simple}.


\knnper{} offers a simple and effective way to address statistical heterogeneity even in a dynamic environment where  client's data distributions change after training. It is indeed sufficient to update the local datastore with new data without the need to retrain the global model. Moreover, each client can independently tune the local kNN to its storage and computing capabilities, partially relieving the most powerful clients from the need to align their model to the weakest ones. 
Finally, \knnper{} has a limited leakage of private information, as personalization only occurs once communication exchanges have ended, and, if needed, it can be easily combined with differential privacy techniques.  


Our contributions are threefold: 1) we propose \knnper{}, a simple personalization mechanism based on local memorization; 2) we provide generalization bounds for the proposed approach in the case of binary classification; 3) through extensive experiments on FL benchmarks, we show that \knnper{}  achieves significantly higher accuracy and fairness than state-of-the-art methods.

The paper is organized as follows. After an overview of related work in Sec.~\ref{sec:related}, we present \knnper{} in Sec.~\ref{sec:method} and provide generalization bounds in Sec.~\ref{sec:generalization_bounds}. Experimental setup and results are described in Sec.~\ref{sec:experimental_setup} and Sec.~\ref{sec:experiment}, respectively.


\section{Related Work}
\label{sec:related}
We discuss personalized FL approaches to address statistical heterogeneity and system heterogeneity as well as nearest neighbours augmented neural networks.

\subsection{Statistical Heterogeneity}
This body of work considers that all clients have the same model architecture but potentially different parameters.

A simple approach for FL personalization is learning first a global model and then fine-tuning its parameters at each client through stochastic gradient descent for a few epochs~\citep{jiang2019improving, yu2020salvaging}; we refer later to this approach as \FedAvg+. \FedAvg+ was later studied by \citet{chen2022on} and \citet{cheng2021fine}. The global model can then be considered as a meta-model to be used as initialization for a few-shot adaptation at each client. Later work~\citep{khodak2019adaptive, fallah2020personalized, emre2021debiasing} has formally established the connection with Model Agnostic Meta Learning (MAML)~\citep{jiang2019improving} and proposed different algorithms to train a more suitable meta-model for local personalization.

\texttt{ClusteredFL}~\citep{sattler2020clustered, ghosh2020efficient, mansour2020three} assumes that clients can be partitioned into several clusters, with clients in the same cluster sharing the same model, while models can be arbitrarily different across clusters.
Clients jointly learn during training the cluster to which they belong as well as the cluster model. \texttt{FedEM} ~\citep{marfoq2021federated} can  be considered as a soft clustering algorithm  as clients learn personalized models as mixtures of a limited number of component models. 

Multi-Task Learning (MTL) allows for more nuanced relations among clients' models by defining federated MTL as a penalized optimization problem, where the penalization term captures clients' dissimilarity.
Seminal work \citep{smith2017federated,Vanhaesebrouck2017a,pmlr-v108-zantedeschi20a} proposed algorithms able to deal with quite generic penalization terms, at the cost of learning only linear models or linear combinations of pre-trained models.
Other MTL-based algorithms~\citep{hanzely2020federated, hanzely2020lower, dinh2020personalized, dinh2021fedu, li2021ditto, huang2021personalized, li2021ditto} are able to train more general models but consider simpler penalization terms (e.g., the distance to the average model).

An alternative approach is to interpolate a global model and one local model per client~\citep{deng2020adaptive,corinzia2019variational,mansour2020three}. 
\citet{zhang2021personalized} extended this idea by letting each client interpolate the local models of other clients with opportune weights learned during training. Our algorithm, \knnper, also interpolates a global and a local model, but the global model plays a double role as it is also used to provide a useful representation for the local kNN.

Closer to our approach, \texttt{FedRep}~\citep{collins2021exploiting}, \texttt{FedPer}~\citep{Arivazhagan2019FederatedLW}, and \pFedGP{} \cite{achituve2021personalized} jointly learn a global latent representation and local models---linear models for \texttt{FedRep} and \texttt{FedPer}, Gaussian processes for \pFedGP---that operate on top of this representation.  
In these algorithms, the progressive refinement of local models affects the shared representation. On the contrary, in \knnper{} only the global model (and then the shared representation) is the object of federated training, and  the shared representation is not influenced by local models, which are learned separately by each client in a second moment. 
Our experiments suggest that \knnper's approach is more efficient. A possible explanation is that jointly learning the shared representation and the local models lead to potentially conflicting and interfering goals.
A similar argument was provided by \citet{li2021ditto} to justify why \texttt{Ditto} replaces, as penalization term, the distance from the average of the local models---as proposed in~\citep{hanzely2020federated, hanzely2020lower, dinh2020personalized}---with the distance from an independently learned global model.
\citet{liang2020think} proposed a someway opposite approach to \texttt{FedRep}, \texttt{FedPer}, and \pFedGP{} by using local representations as input to a global model, but the representations and the global model are still jointly learned.  
An additional advantage of \knnper's clear separation between global and local model training is that, because each client does not share any information about its local model with the server, the risk of leaking private information is reduced. 
In particular, \knnper{} enjoys the same privacy guarantees as \FedAvg{}, and can be easily combined with differential privacy techniques \citep{wei2020federated}. 

To the best of our knowledge,  \pFedGP{} and \knnper{} are the first attempts to learn semi-parametric models~\citep{bickel1998efficient} in a federated setting. \pFedGP{} relies on Gaussian processes and then has higher computational cost than \knnper{} both at training and inference.

\subsection{System Heterogeneity}
Some FL application scenarios envision  clients with highly heterogeneous hardware, like smartphones, IoT devices, edge computing servers, and the cloud.
Ideally, each client could learn a potentially different model architecture, suited to its capabilities. Such system heterogeneity has been studied much less than statistical heterogeneity.   

Some work~\citep{lin2020ensemble, li2019fedmd, zhu2021datafree, zhang2021fedzkt} proposed to address system heterogeneity by distilling the knowledge from a global teacher to clients' student models with different architectures. While early methods \citep{li2019fedmd, lin2020ensemble} required the access to an extra (unlabeled) public dataset, more recent ones \citep{zhu2021datafree, zhang2021fedzkt} eliminated this requirement. 

Some papers \citep{diao2020heterofl,horvath2021fjord,pilet2021} propose that each client only trains a sub-model  of a global model. The sub-model size is determined by the client's computational capabilities. The approach appears particularly advantageous for convolutional neural networks with clients selecting only a limited subset of channels.

\citet{tan2021fedproto} followed another approach where devices and server communicate prototypes, i.e., average representations for all samples in a given class, instead of communicating model's gradients or parameters, allowing each client to have a different model architecture and input space.

While in this paper we assume that \knnper{} relies on a shared global model, it is possible to replace it with  heterogeneous models adapted to the clients' capabilities and jointly trained following one of the methods listed above. Then, \knnper's  interpolation with a kNN model extends these methods to address not only system heterogeneity, but also statistical heterogeneity. 

To the best of our knowledge, the only existing method that takes into account both system and statistical heterogeneity is \texttt{pFedHN}~\citep{Shamsian2021PersonalizedFL}. \texttt{pFedHN} feeds local clients representations to a global (across clients) hypernetwork, which can output personalized heterogeneous models. Unfortunately, the hypernetwork has a large memory footprint already for  small clients' models (e.g., the hypernetwork in the experiments in~\citep{Shamsian2021PersonalizedFL} has 100 more parameters than the output model): it is not clear if  \texttt{pFedHN} can scale to complex models. 

We observe that  \knnper's  kNN model can itself be adapted to  client's capabilities by tuning the size of the datastore and/or selecting an appropriate approximate kNN algorithm, like FAISS \citep{faiss}, HNSW \citep{malkov2020efficient}, or ProtoNN~\citep{gupta17protonn} for IoT resource-scarce devices, e.g., based on Arduino.

\subsection{Nearest Neighbours Augmented Neural Networks}
Recent work proposed to augment neural networks with nearest neighbours classifiers for  applications to language modelling~\citep{khandelwal2019generalization}, neural machine translation~\citep{khandelwal2020nearest}, computer vision~\citep{papernot2018deep, orhan2018simple}, and few-shot learning~\citep{snell2017prototypical, wang2019simpleshot}.

In~\citep{khandelwal2019generalization,khandelwal2020nearest} kNN improves model performance by memorizing explicitly (rather than implicitly in model parameters) rare patterns. \citet{papernot2018deep} and \citet{orhan2018simple} showed that memorization can also increase the robustness of models against adversarial attacks. Differently from these lines of work, our paper shows that local memorization at each FL client
can capture the client’s local distribution shift with respect to the global distribution. In this sense, our use of kNN is more similar to what is proposed for few shot learning in~\citep{snell2017prototypical,wang2019simpleshot}, where an embedding function is learned for future application to new small datasets. Beside the different learning problem, the algorithms proposed  in~\citep{snell2017prototypical,wang2019simpleshot} do not rely on models' interpolation and do not enjoy generalization guarantees as \knnper{} does. Moreover, their natural extension to the FL setting would lead to jointly learn the  shared representation and the local kNN models, a potentially less efficient approach than ours as we discussed above when presenting \texttt{FedRep}, 
and \pFedGP.

\section{\knnper{} Algorithm}
\label{sec:method}
In this work we consider $M$ classification or regression tasks also called clients. 
Each client $m \in [M]$ has a local dataset $\mathcal{S}_{m}=\left\{s_{m}^{(i)}=\left(\vec{x}_{m}^{(i)},~y_{m}^{(i)}\right),1  \leq i \leq n_{m}\right\}$ with $n_m$ samples drawn i.i.d.~from a distribution $\mathcal{D}_{m}$  over the domain $\mathcal{X} \times \mathcal{Y}$. 
Local data distributions $\left\{\mathcal{D}_{m}\right\}_{m\in[M]}$ are in general different, thus it is natural to fit a separate model (hypothesis) $h_{m} \in \mathcal{H}$ to each data distribution $\mathcal{D}_{m}$.  We consider that each hypothesis $h\in \mathcal{H}$ is a discriminative model mapping each input $\vec{x} \in \mathcal{X}$ to a probability distribution over the set~$\mathcal Y$, i.e., $h: \mathcal{X} \mapsto \Delta^{|\mathcal{Y}|}$, where $\Delta^{C}$ denotes the unitary simplex of dimension $C$. A hypothesis then can be interpreted as (an estimation of) a conditional probability distribution $\mathcal{D}(y|\vec{x})$.

Personalized FL aims to solve (in parallel) the following optimization problems
\begin{equation}
    \label{eq:main_problem}
    \forall m \in [M],\quad h_{m}^{*} \in \argmin_{h \in \mathcal{H}} \mathcal{L}_{\mathcal{D}_{m}}(h),
\end{equation}
where $[M]$ denotes the set of positive integers up to $M$, $l: \Delta^{|\mathcal{Y}|} \times \mathcal{Y} \mapsto \mathbb{R}^{+}$ is the loss function,\footnote{
    In the case of (multi-output) regression, we have $h_{m}: \mathcal{X} \mapsto \mathbb{R}^{d}$ for some $d\geq 1$ and $l: \mathbb{R}^{d}\times\mathbb{R}^{d} \mapsto \mathbb{R}^{+}$.
}
and $\mathcal{L}_{\mathcal{D}_{m}}(h_{m})=\E_{(\vec{x},y) \sim \mathcal{D}_{m}} \left[l(h_{m}\left(\vec{x}), y\right)\right]$ is the true risk of a model $h_{m}$ under data distribution $\mathcal{D}_{m}$. 

We suppose that all tasks have access to a global discriminative model $h_{\mathcal{S}}$
minimizing the empirical risk on the aggregated dataset $\mathcal{S} \triangleq \bigcup_{m=1}^{M} \mathcal{S}_{m}$, i.e.,
\begin{flalign}
    \label{eq:ERM_min}
    h_{\mathcal{S}} \in \argmin_{h\in\mathcal{H}} {\mathcal{L}}_{\mathcal{S}}\left(h\right) , 
\end{flalign}
where $\mathcal{L}_{\mathcal{S}}\left(h\right) \triangleq \sum_{m=1}^{M}\frac{n_{m}}{n} \cdot \frac{1}{n_{m}}\sum_{i=1}^{n_{m}} l\left(h\left(\vec{x}_{m}^{(i)}\right), y_{m}^{(i)}\right)$, and $n=\sum_{m=1}^{M}n_{m}$.
Typically $h_{\mathcal{S}}$ is a feed-forward neural network, jointly trained by the clients using a standard FL algorithm like~\FedAvg. 

We also suppose that the global model can be used to compute a fixed-length representation for any input $\vec{x} \in \mathcal{X}$, and we use $\phi_{h_{\mathcal{S}}}:\mathcal{X} \mapsto \mathbb{R}^{p}$ to denote the function that maps the input  $\vec{x} \in \mathcal{X}$ to its representation.
The intermediate representation can be, for example, the output of the last convolutional layer in the case of CNNs, or the last hidden state in the case of recurrent networks or the output of an arbitrary self-attention layer in the case of transformers. Note that an alternative possible approach would be to separately learn an independent shared representation, e.g., using metric learning techniques~\cite{bellet15}.

Our method (see \cref{alg:knnper}) involves augmenting the global model with a local nearest neighbors' retrieval mechanism at each client. The proposed method does not need any additional training; it only requires a single forward pass over the local dataset $S_{m},~m\in[M]$: client $m$ computes the intermediate representation $\phi_{h_{\mathcal{S}}}\!\left(\vec{x}\right)$ for each sample $(\vec{x}, y)\in \mathcal{S}_{m}$. The corresponding representation-label pairs are stored in a local key-value datastore $\left(\mathcal{K}_{m}, \mathcal{V}_{m}\right)$ that is queried during inference. Formally,
\begin{equation}
    \left(\mathcal{K}_{m}, \mathcal{V}_{m}\right) = \left\{\left(\phi_{h_{\mathcal{S}}}\!\left(\vec{x}^{(i)}_{m}\right), y^{(i)}_{m}\right), \forall \left(\vec{x}^{(i)}_{m}, y^{(i)}_{m}\right) \in \mathcal{S}_{m} \right\}.
\end{equation}

At inference time, given input data $\vec{x} \in \mathcal{X}$, client $m\in[M]$  computes  $h_{\mathcal{S}}\!\left(\vec{x}\right)$ and the intermediate representation $\phi_{h_{\mathcal{S}}}\!\left(\vec{x}\right)$.
Then, it queries its local datastore $\left(\mathcal{K}_{m}, \mathcal{V}_{m}\right)$ with $\phi_{h_{\mathcal{S}}}\!\left(\vec{x}\right)$ to retrieve its $k$-nearest neighbors $\mathcal{N}_{m}^{(k)}\!\left(\vec{x}\right)$ according to a distance $d\left(\cdot,\cdot\right)$:
\begin{equation}
    \mathcal{N}_{m}^{(k)}\!\left(\vec{x}\right) = \Bigg(\phi_{h_{\mathcal{S}}}\big(\vec{x}_{\pi^{(i)}_{m}\left(\vec{x}\right)}\big), y_{\pi^{(i)}_{m}\left(\vec{x}\right)}\Bigg)_{1 \leq i \leq k},
\end{equation}
where $\pi^{(1)}_{m}\left(\vec{x}\right), \dots, \pi^{(n_m)}_{m}\left(\vec{x}\right)$ is a permutation of $[n_{m}]$ corresponding to the distance of the samples in $\mathcal S_m$ from  $\vec{x}$, i.e., for $i\in[n_{m}-1]$,
\begin{flalign}
    d\Big(\phi_{h_{\mathcal{S}}}\big(\vec{x}\big), & \phi_{h_{\mathcal{S}}}\big(\vec{x}_{\pi^{(i)}_{m}\left(\vec{x}\right)}\big)\Big) \leq  \nonumber
    \\
    & \quad d\Big(\phi_{h_{\mathcal{S}}}\big(\vec{x}\big), \phi_{h_{\mathcal{S}}}\big(\vec{x}_{\pi^{(i+1)}_{m}\left(\vec{x}\right)}\big)\Big) . 
\end{flalign}

Then, the client computes a local hypothesis $h^{(k)}_{\mathcal{S}_{m}}$ which estimates the conditional probability $\mathcal D_m(y|\vec{x})$ using a kNN method, e.g., with a Gaussian kernel:
\begin{flalign}
    \label{eq:gaussian_kernel}
    \left[ h^{(k)}_{\mathcal{S}_{m}}(\vec{x})\right]_y &
    \propto \sum_{i=1}^{k} \mathds{1}_{\left\{y=y_{\pi^{(i)}_{m}\left(\vec{x}\right)}\right\}} \times \nonumber
    \\
    & \;\;\; \exp\left\{-d\left(\phi_{h_{\mathcal{S}}}\!\left(\vec{x}\right),\phi_{h_{\mathcal{S}}}\big(\vec{x}_{\pi^{(i)}_{m}\left(\vec{x}\right)}\big) \right)\right\}.
\end{flalign}



The final decision rule (hypothesis) at client $m\in[M]$ ($h_{m,\lambda_{m}}$) is obtained interpolating the nearest neighbour distribution $h^{(k)}_{\mathcal{S}_{m}}$ with the distribution obtained from the global model  $h_{\mathcal{S}}$ using a hyper-parameter $\lambda_{m} \in (0,1)$ to produce the final prediction, i.e.,
\begin{equation}
    \label{eq:main_model}
    {h}_{m, \lambda_{m}}\left(\vec{x}\right) \triangleq \lambda_{m} \cdot  h^{(k)}_{ \mathcal{S}_{m}}\left(\vec{x}\right) + (1- \lambda_{m}) \cdot h_{\mathcal{S}}\!\left(\vec{x}\right). 
\end{equation}
As ${h}_{m, \lambda_{m}}$ may not belong to $\mathcal H$, we are considering an \emph{improper learning} setting. 
The parameter $\lambda_m$ is tuned at client~$m$ through a local validation dataset or cross-validation as in~\cite{corinzia2019variational,mansour2020three,zhang2021personalized, li2021ditto}. 
Clients could also use different values $k_m$ and different distance metrics $d_{m}(\cdot)$, but, in what follows, we consider them equal across clients. Also our experiments 
in Sec.~\ref{sec:experiment} show that $k$ and $d(\cdot)$ do not require careful tuning.


\begin{algorithm}[tb]
    \caption{\knnper{} (Typical usage)}
    \label{alg:knnper}
    \begin{algorithmic}
        \STATE Learn global model using available clients with \FedAvg.
        \FOR{each client $m\in[M]$ (in parallel)}
            \STATE Build datastore using $\mathcal{S}_{m}$.
            \STATE At inference on  $\vec{x} \in \mathcal{X}$, return $h_{m, \lambda_{m}}\left(\vec{x}\right)$ given by \eqref{eq:main_model}
        \ENDFOR
    \end{algorithmic}
\end{algorithm}

\section{Generalization Bounds}
\label{sec:generalization_bounds}
In this section we provide a generalization bound associated with the proposed approach in the case of binary classification, namely $\mathcal{Y}=\left\{0, 1\right\}$, when only one neighbour is used for kNN estimation, i.e., $k=1$, and $d\left(\cdot, \cdot\right)$ is the Euclidean distance. 
For client $m\in[M]$, we denote by $\eta_{m}:\mathcal{X} \mapsto \mathbb{R}$ the true conditional probability of label $1$, that is
\begin{equation}
    \eta_{m}\left(\vec{x}\right) = \mathcal{D}_{m}\left(y=1|\vec{x}\right).
\end{equation}

Our result holds under the following assumptions:
\begin{assumption}[Bounded representation]
    \label{assum:bounded_loss}
    $\phi_{{h}_{\mathcal S}}: \mathcal X \mapsto [0, 1]^{p}$.
\end{assumption} 

\begin{assumption}[Bounded loss]
    \label{assum:classification_error}
     $l:  \Delta^{|\mathcal{Y}|} \times \mathcal{Y} \mapsto  [0,1]$. Moreover, 
    for $y, y' \in \left\{0, 1\right\}$, $l(\vec{e}_{y}, y') = \mathds{1}_{y\neq y'}$,
    where $\vec{e}_{y} \in \Delta^{|\mathcal{Y}|}$ is the vector having all entries equal to $0$ except the entry on the $y$-th coordinate.
\end{assumption}

\begin{remark}
    Loss boundedness is a common assumption, e.g., \cite{mansour2020three},\citep[Ch.~4]{shalev2014understanding}. The second requirement is that the maximum loss is achieved when the model is fully confident about a prediction, but this is wrong.
    A simple transformation of common loss functions---e.g., exponentiating the logistic function---make them satisfy  Assumption~\ref{assum:classification_error}. 
\end{remark}

\begin{assumption}[Loss convexity] 
    \label{assum:convex_loss}
    The loss function is convex on the first variable
    \begin{flalign}
        \forall y_{1}, y_{2} & \in \Delta^{|\mathcal{Y}|}, \forall y \in \mathcal{Y},~\forall \lambda_{m} \in [0,1], \nonumber
        \\
        &  l(\lambda_{m} \cdot y_{1} + (1-\lambda_{m})\cdot y_{2}, y) \leq \nonumber
        \\
        & \qquad \qquad \quad \lambda_{m} \cdot l( y_{1}, y) + (1-\lambda_{m})\cdot l(y_{2}, y).
    \end{flalign}
\end{assumption}

\begin{remark}
    Assumption~\ref{assum:convex_loss} holds for most loss functions used in supervised machine learning, including the \emph{mean squared error} loss, the  \emph{cross-entropy} loss, and  the \emph{hinge loss}.
\end{remark}


\begin{assumption} 
    \label{assum:lipschitz}
    There exist constants $\gamma_{1}, \gamma_{2} > 0$, such that for any dataset $\mathcal{S}$ drawn from $\mathcal{X} \times \mathcal{Y}$ and any data points $\vec{x}, \vec{x}' \in \mathcal{X}$, we have 
    \begin{flalign}
        \big|\eta_{m}\left(\vec{x}\right)  - & \eta_{m} \left(\vec{x'}\right) \big| \leq  d\left(\phi_{h_{\mathcal{S}}}\left(\vec{x}\right) , \phi_{h_{\mathcal{S}}}\left(\vec{x}'\right) \right) \times  \nonumber
        \\ 
        & \quad \left(\gamma_{1} + \gamma_{2} \left(\mathcal{L}_{\mathcal{D}_{m}}\left({h}_{\mathcal{S}} \right) - \mathcal{L}_{\mathcal{D}_{m}}\left(h_{m}^{*} \right)\right) \right) ,
    \end{flalign}
    where $h^{*}_{m} \in \argmin_{h\in\mathcal{H}}\mathcal{L}_{D_{m}}\left(h\right)$.
\end{assumption}

This assumption means that if two samples $\vec{x}$ and $\vec{x'}$ have close representations $\phi_{h_{\mathcal{S}}}\left(\vec{x}\right)$ and $\phi_{h_{\mathcal{S}}}\left(\vec{x}'\right)$, then their labels are likely to be the same ($|\eta_m(\vec{x}) - \eta_m(\vec{x'}) |$ is small). This is all the more so, the better 
$h_{\mathcal{S}}$ predictions are for  distribution~$\mathcal{D}_{m}$, $m \in [M]$ (the smaller $\mathcal{L}_{\mathcal{D}_{m}}\left({h}_{\mathcal{S}} \right) - \mathcal{L}_{\mathcal{D}_{m}}\left(h_{m}^{*} \right)$~is). Experimental results support Assumption \ref{assum:lipschitz} (see \cref{f:base_model_effect_main}).

Our generalization bound depends, as usual, on the complexity of the hypothesis class $\mathcal H$ (expressed by its VC-dimension, $d_{\mathcal H}$) and on the size of the local and global datasets ($n_m$ and $n$, respectively), but also on the distance between the local distribution $\mathcal D_m$ and the average distribution $\bar{\mathcal{D}} = \sum_{m=1}^{M} \frac{n_{m}}{n} \cdot \mathcal{D}_{m}$, which is the one the global model $h_{\mathcal S}$ is targeting (see~\eqref{eq:ERM_min}). The distance between two distributions $\mathcal D$ and $\mathcal{D}'$ associated to a hypothesis class $\mathcal{H}$ can be quantified by the \emph{label discrepancy} \citep{mansour2020three}:
\begin{equation}
        \disc_{\mathcal{H}}\left(\mathcal{D}, \mathcal{D}'\right) = \max_{h\in\mathcal{H}} \left|\mathcal{L}_{\mathcal{D}}\left(h\right) - \mathcal{L}_{\mathcal{D}'}\left(h\right) \right|.
    \end{equation}

\begin{thm}
    \label{thm:generalization_bound}
    Suppose that Assumptions~\ref{assum:bounded_loss}--\ref{assum:lipschitz} hold, and consider $m\in[M]$ and $\lambda_{m} \in (0,1)$, then there exist constants $c_{1}, c_{2}, c_{3}$, $c_{4}$, and $c_5 \in\mathbb{R}$, such that

    \begin{flalign}
        \label{eq:bound}
        & \E_{\mathcal{S}\sim \otimes_{m=1}^{M}\mathcal{D}_{m}^{n_{m}}}  \left[\mathcal{L}_{\mathcal{D}_{m}} \left(h_{m, \lambda_{m}}\right)\right]  \leq  \left(1 + \lambda_{m}\right)\cdot \mathcal{L}_{\mathcal{D}_{m}} \left({h}^{*}_m\right)  \nonumber
        \\ 
        &  \quad + c_{1}\left(1-\lambda_{m}\right)  \cdot \disc_{\mathcal{H}}\left(\bar{\mathcal{D}}, \mathcal{D}_{m}\right) \nonumber
        \\
        &  \quad + c_{2}\lambda_{m} \cdot \frac{\sqrt{p}}{\sqrt[p+1]{n_{m}}} \cdot  \left(\disc_{\mathcal{H}}\left(\bar{\mathcal{D}}, \mathcal{D}_{m}\right) + 1 \right) \nonumber
        \\
        & \quad  + c_{3}\left(1-\lambda_{m}\right) \cdot \sqrt{\frac{d_{\mathcal H}}{n}} \cdot \sqrt{c_{4} + \log\left(\frac{n}{d_{\mathcal H}}\right)} \nonumber
        \\
        & \quad + c_{5} \lambda_{m} \cdot \sqrt{\frac{d_{\mathcal H}}{n}}  \cdot \sqrt{c_{4} + \log\left(\frac{n}{d_{\mathcal H}}\right)} \cdot \frac{\sqrt{p}}{\sqrt[p+1]{n_{m}}}, 
    \end{flalign}
    where $d_{\mathcal H}$ is the the VC dimension of the hypothesis class~$\mathcal{H}$, $n=\sum_{m=1}^{M}n_{m}$, $\bar{\mathcal{D}} = \sum_{m=1}^{M} \frac{n_{m}}{n} \cdot \mathcal{D}_{m}$, $p$ is the dimension of representations, and $\disc_{\mathcal{H}}$ is the label discrepancy associated to the hypothesis class $\mathcal{H}$.
\end{thm}
The proof of Theorem~\ref{thm:generalization_bound} is in Appendix~\ref{app:proofs}. Let us consider, for simplicity, the non-agnostic case, i.e, $\mathcal{L}_{\mathcal{D}_{m}}\left(h_{m}^{*}\right)=0$. We observe that, when clients  only use the global model ($\lambda_{m}=0$), our generalization bound is analogous to the probabilistic bound in  \citep[Eq.~(2)]{mansour2020three}. In particular, if data is i.i.d.~distributed across the clients ($\disc_{\mathcal{H}}\left(\bar{\mathcal{D}}, \mathcal{D}_{m}\right)=0$), the difference between the expected losses of the learned model and  the optimal one decreases with rate $\tilde{\mathcal O}\left(\sqrt{\frac{d_{\mathcal H}}{n}}\right)$. Instead, when each client only uses the kNN model ($\lambda_{m}=1$),\footnote{Note that the kNN model still relies on the representation provided by the global model.} we recover the kNN generalization bound in \citep[Thm~19.3]{shalev2014understanding}.

The bound~\eqref{eq:bound} leads to predict that client $m$ should give a larger weight ($\lambda_m > 1/2$) to its kNN model, when  $n_m$ exceeds a given threshold, even when local distributions are identical. The bound contributes then to explain why adding a memorization mechanism on top of a pretrained model can improve performance, as observed in \citep{khandelwal2019generalization} and  \citep{khandelwal2020nearest}.
While it is difficult to quantify the threshold analytically (also because the constants involved depend on $\gamma_1$ and $\gamma_2$ in Assumption~\ref{assum:lipschitz}), our experiments in Sec.~\ref{sec:experiment} show that even clients with a few tens of samples weigh more the kNN model than the global one.

\section{Experimental Setup}
\label{sec:experimental_setup}
\begin{table*}[t]
    \caption{Datasets and models.}
    \label{tab:datasets_models} 
    \vskip 0.15in
    \begin{center}
    \begin{small}
    \begin{sc}
    \begin{tabular}{ l  l  r  r  l}
    \toprule
        \textbf{Dataset} & \textbf{Task}  &  \textbf{Clients} & \textbf{Total samples} & \textbf{Model} \\
         \midrule
        FEMNIST & Handwritten character recognition & $3,550$ & $805,263$ & MobileNet-v2 \\
         CIFAR-10  & Image classification  & $200$ & $60,000$ & MobileNet-v2  \\
         CIFAR-100  & Image classification  & $200$ & $60,000$ & MobileNet-v2   \\
         Shakespeare  & Next-Character Prediction  & $778$  & $4,226,158$ & Stacked-LSTM \\
    \bottomrule
    \end{tabular}%
    \end{sc}
    \end{small}
    \end{center}
    \vskip -0.1in
\end{table*}

We evaluate \knnper{} on four federated datasets spanning a wide range of machine learning tasks: language modeling (Shakespeare \citep{caldas2018leaf, mcmahan2017communication}), image classification (CIFAR-10 and CIFAR-100 \citep{Krizhevsky09learningmultiple}), handwritten character recognition (FEMNIST \citep{caldas2018leaf}). Unless otherwise said, \knnper{}'s global model $h_{\mathcal{S}}$ is trained by all clients through \FedAvg. Code is available at \url{https://github.com/omarfoq/knn-per}.


\textbf{Datasets.} For Shakespeare and FEMNIST datasets there is a natural way to partition data through clients (by character and by writer, respectively).
We relied on common approaches in the literature to sample heterogenous local datasets from CIFAR-10 and CIFAR-100. We created a federated version of CIFAR-10  by randomly partitioning the dataset among  clients using a symmetric Dirichlet distribution, as done in \cite{Wang2020Federated}. In particular, for each label $y$ we sampled a vector $p_y$ from a Dirichlet distribution of order $M=200$ and parameter $\alpha=0.3$ (unless otherwise specified) and allocated to client $m$ a $p_{y,m}$ fraction of all training instances of class $y$. The approach ensures that the number of data points and label distributions are unbalanced across clients. For CIFAR-100, we exploit the availability  of ``coarse" and ``fine" label structure, to partition the dataset  using pachinko allocation method \citep{li2006pachinko} as in \citep{reddi2021adaptive}. The method generates local datasets with heterogeneous distributions by combining a per-client Dirichlet distribution with parameter $\alpha=0.3$ (unless otherwise specified) over the coarse labels and a per-coarse-label Dirichlet distribution with parameter $\beta=10$ over the corresponding fine labels. 
We also partitioned CIFAR-10 and CIFAR-100 
in a different way
following \citep{achituve2021personalized}: each  client has only samples from two and ten classes for CIFAR-10 and CIFAR-100, respectively.
We refer to the resulting datasets as CIFAR-10 (v2) and  CIFAR-100 (v2).
For FEMNIST and Shakespeare, we randomly split each local dataset into training ($60\%$), validation ($20\%$), and  test ($20\%$) sets. For CIFAR-10 and CIFAR-100, we maintained the original training/test data split and used $20\%$ of the training dataset as validation dataset. Table~\ref{tab:datasets_models} summarizes datasets, models and number of clients.

\textbf{Models and representations.} For CIFAR-100, CIFAR-10, and FEMNIST, we used MobileNet-v2 \citep{sandler2018mobilenetv2} as a base model with the output of the last hidden layer---a $1280$-dimensional vector---as representation. For Shakespeare, the base model was a stacked LSTM model with two layers, each of them with $256$ units; a $1024$-dimensional representation was obtained by concatenating the hidden states and the cell states. 

\textbf{Benchmarks.} We compared \knnper{} with locally trained models (with no collaboration across clients) and  \FedAvg~\citep{mcmahan2017communication}, as well as with one method for each of the personalization approaches described in Sec.~\ref{sec:related}, namely, 
\FedAvg+~\citep{jiang2019improving},\footnote{
    We also implemented the more sophisticated first-order MAML approach from~\citep{fallah2020personalized}, but had worse performance than \FedAvg+.}
\texttt{ClusteredFL}~\citep{sattler2020clustered},
\texttt{Ditto}~\citep{li2021ditto},
\texttt{FedRep}~\citep{collins2021exploiting}, \texttt{APFL}~\citep{deng2020adaptive}, and \pFedGP{}~\citep{achituve2021personalized}.\footnote{
    We were able to run the official \pFedGP{}'s code (\url{https://github.com/IdanAchituve/pFedGP}) only on datasets partitioned as in~\citep{achituve2021personalized}. 
}
For each method, and each dataset, we tuned the learning rate via grid search on the values $\left\{10^{-0.5}, 10^{-1}, 10^{-1.5}, 10^{-2}, 10^{-2.5}\right\}$.   \texttt{FedPer}'s learning rate for network heads' training was separately tuned on the same grid. \texttt{Ditto}'s penalization parameter $\lambda_{m}$ was selected among the values $\left\{10^{1}, 10^{0}, 10^{-1}, 10^{-2}\right\}$ on a per-client basis.  For \texttt{ClusteredFL}, we used the same values of tolerance specified in its official implementation \citep{sattler2020clustered}. We found tuning \texttt{tol1} and \texttt{tol2} particularly hard: no empirical rule is provided in \citep{sattler2020clustered}, and the few random settings we tried did not show any improvement in comparison to the default ones. For \texttt{APFL}, the mixing parameter $\alpha$ was tuned via grid search on the grid $\left\{0.1, 0.3, 0.5, 0.7, 0.9\right\}$.
For \pFedGP{}, we used the same hyperparameters as in \citep{achituve2021personalized}. 
The parameter~$\lambda_{m}$ of \knnper{} was tuned for each client via grid search on the grid $\left\{0.0, 0.1, 0.3, 0.5, 0.7, 0.9, 1.0\right\}$, and the number of neighbours was set to $k=10$.
Once the optimal hyperparameters' values were selected, models were retrained on the concatenation of training and validation sets.


\begin{table*}[t]
    \caption{Test accuracy: average across clients\,/\,bottom decile.}
    \label{tab:results_summary}
    \begin{center}
    \begin{small}
    \begin{sc}
    \resizebox{\textwidth}{!} 
    { 
    \begin{tabular}{l c c c c c c c c c}
        \toprule
        \multirow{2}{*}{Dataset} & \multirow{2}{*}{Local} & \multirow{2}{*}{\FedAvg}  & \multirow{2}{*}{\FedAvg+}  & \multirow{2}{*}{\texttt{ClusteredFL}} & \multirow{2}{*}{\texttt{Ditto}}  & \multirow{2}{*}{\texttt{FedRep}} & \multirow{2}{*}{\texttt{APFL}} & \multirow{2}{*}{\pFedGP{}} & \texttt{\knnper}
        \\
        & & & & &  & & & & (Ours)
        \\
        \midrule
            FEMNIST & $71.0 \,/\, 57.5$ & $83.4 \,/\, 68.9$ & $84.3 \,/\, 69.4$ & $83.7 \,/\, 69.4$ & $84.3 \,/\, 71.3$ &  $85.3 \,/\, 72.7$ & $84.1 \,/\, 69.4$ & $ - \,/\, -$ & $\mathbf{88.2} \,/\, \mathbf{78.8}$ 
            \\
            CIFAR-10  & $57.6 \,/\, 41.1$ & $72.8 \,/\, 59.6$ & $75.2 \,/\, 62.3$ & $73.3 \,/\, 61.5$ &  $80.0 \,/\, 66.5$  & $77.7 \,/\, 65.2$  & $78.9 \,/\, 68.1$ & $ - \,/\, -$ & $\mathbf{83.0} \,/\, \mathbf{71.4}$
            \\
            CIFAR-10~(v2) & $82.4 \,/\, 71.3$ & $67.9 \,/\, 60.1$ & $85.0 \,/\, 79.6$ & $79.9 \,/\, 72.3$ &  $86.3 \,/\, 80.6$ & $89.1 \,/\, 85.3$  & $82.6 \,/\, 76.4$ & $ 88.9 \,/\, 84.1$ & $\mathbf{93.8} \,/\, \mathbf{88.2}$
            \\
            CIFAR-100 & $31.5 \,/\, 19.8$ & $47.4 \,/\, 36.0$ & $51.4 \,/\, 41.1$ & $47.2 \,/\, 36.2$  & $52.0 \,/\, 41.4$ & $53.2 \,/\, 41.7$ & $51.7 \,/\, 41.1$ & $ - \,/\, -$ & $\mathbf{55.0} \,/\, \mathbf{43.6}$
            \\
            CIFAR-100~(v2) & $45.7 \,/\, 38.2$ & $ 42.3\,/\, 34.8$ & $48.1 \,/\ 41.9 $ & $43.5 \,/\, 37.2$ &  $48.7 \,/\, 40.3$ & $70.1 \,/\, 65.2$  & $48.3 \,/\, 42.1$ & $ 61.1 \,/\, 50.0$ & $\mathbf{74.6} \,/\, \mathbf{67.3}$
            \\
            Shakespeare & $32.0 \,/\, 16.0$ & $48.1 \,/\, 43.1$ & $47.0 \,/\, 42.2$ & $46.7 \,/\, 41.4$ & $47.9 \,/\, 42.6$ & $47.2 \,/\, 42.3$ & $45.9 \,/\, 42.4$ & $ - \,/\, -$ & $\mathbf{51.4} \,/\, \mathbf{45.4}$
            \\
        \bottomrule
    \end{tabular}
    }
    \end{sc}
    \end{small}
    \end{center}
\end{table*}

\begin{table*}[t]
    \caption{Average test accuracy across clients unseen at training  (train accuracy between parentheses).}
    \label{tab:results_summary_unsen_clients}
    \begin{center}
    \resizebox{\textwidth}{!} 
    { 
    \begin{tabular}{l  c c c c c c c c }
        \toprule
        \multirow{2}{*}{Dataset}  & \multirow{2}{*}{\FedAvg}  & \multirow{2}{*}{\FedAvg+}  & \multirow{2}{*}{\texttt{ClusteredFL}} & \multirow{2}{*}{\texttt{Ditto}}  & \multirow{2}{*}{\texttt{FedRep}} & \multirow{2}{*}{\texttt{APFL}}& \multirow{2}{*}{\pFedGP{}} &   \knnper
        \\
        & & & &  & &  & & (Ours)
        \\
        \midrule
            FEMNIST  & $83.1$ ($83.3$) & $84.2$ ($88.5$) & $83.2$ ($86.0$) & $83.9$ ($86.9$) &  $85.4$ ($88.9$) & $84.2$ ($85.5$) & $-$ & $\mathbf{88.1}$ ($90.5$) 
            \\
            CIFAR-10  & $72.9$ ($72.8$) & $75.3$ ($78.2$)  & $73.9$ ($76.2$) & $79.7$ ($84.3$) &  $76.4$ ($79.5$)  & $79.2$ ($80.6$) & $-$ & $\mathbf{82.4}$ ($87.1$)
            \\
            CIFAR-10 (v2) &  $67.5$ ($68.1$) & $85.1$ ($85.0$)  & $79.6$ ($79.9$) & $85.9$ ($86.0$) &  $89.0$ ($89.1$) & $82.3$ ($82.5$) & $89.0$ ($88.8$) & $\mathbf{93.0}$ ($93.1$)
            \\
            CIFAR-100 & $47.1$ ($47.5$) & $50.8$ ($53.4$) & $47.1$ ($48.2$) & $52.1$ ($57.3$) &  $53.5$ ($58.2$) & $49.1$ ($52.7$) & $-$ & $\mathbf{56.1}$ ($59.3$)
            \\
            CIFAR-100 (v2) &  $42.1$ ($42.2$) & $47.9$ ($48.1$)  & $43.2$ ($43.4$) & $48.8$ ($48.5$) &  $69.8$ ($70.0$) & $48.2$ ($48.4$) & $61.3$ ($61.0$) & $\mathbf{74.3}$ ($74.5$)
            \\
            Shakespeare & $49.0$ ($48.3$) & $49.3$ ($48.1$) & $49.4$ ($46.7$) & $48.1$ ($49.2$) &  $48.7$ ($47.8$) & $46.1$ ($52.7$) & $-$ & $\mathbf{50.7}$ ($64.2$)
            \\
        \bottomrule
    \end{tabular}
    }
    \end{center}
\end{table*}

\begin{figure*}[t]
    \centering
    \begin{subfigure}[b]{0.35\textwidth}
        \centering 
        \includegraphics[width=\textwidth, height=0.8\textwidth]{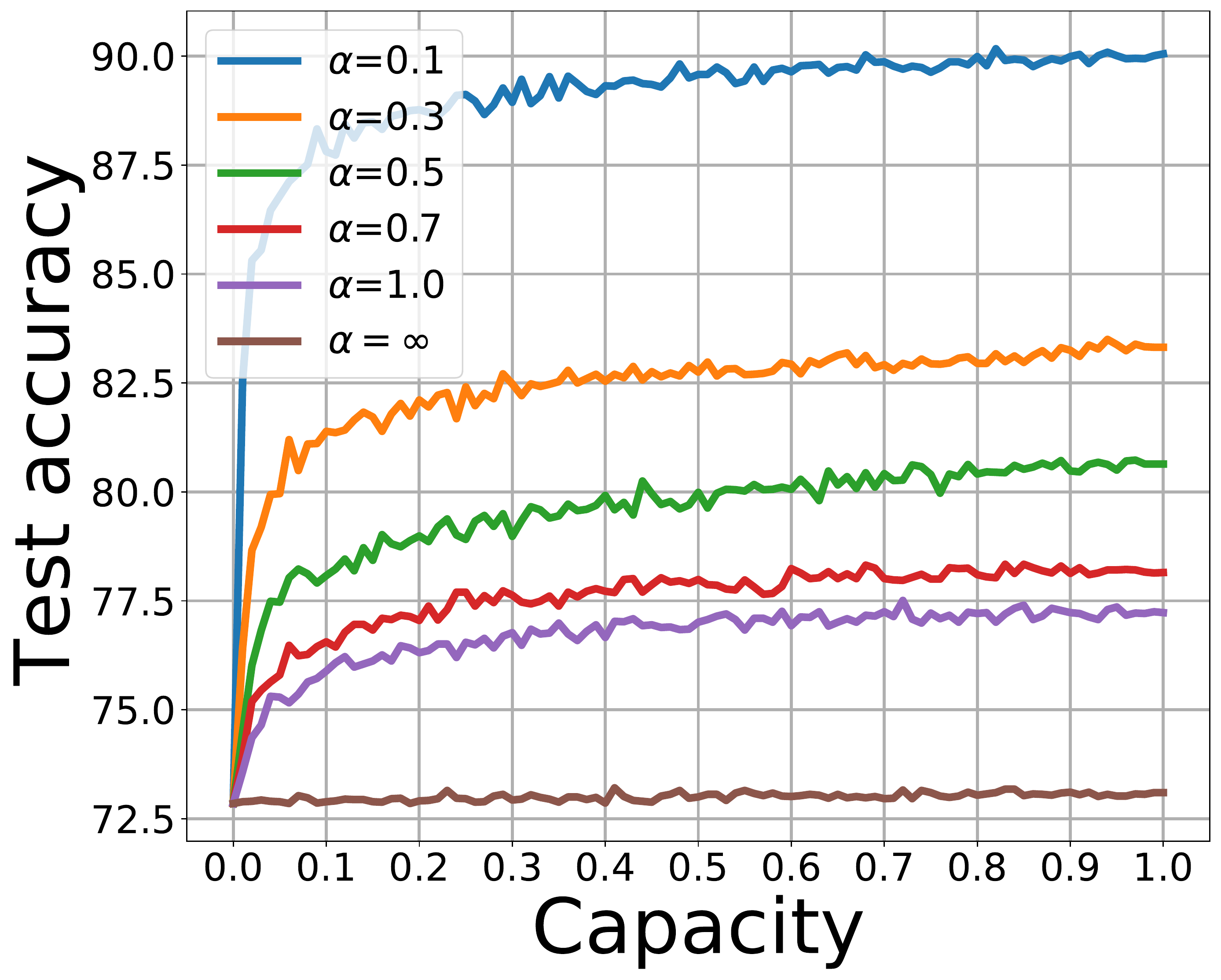}
        \subcaption[]{\small CIFAR-10}
    \end{subfigure}
    \hspace{0.1\textwidth}
    \begin{subfigure}[b]{0.35\textwidth}
        \centering
        \includegraphics[width=\textwidth, height=0.8\textwidth]{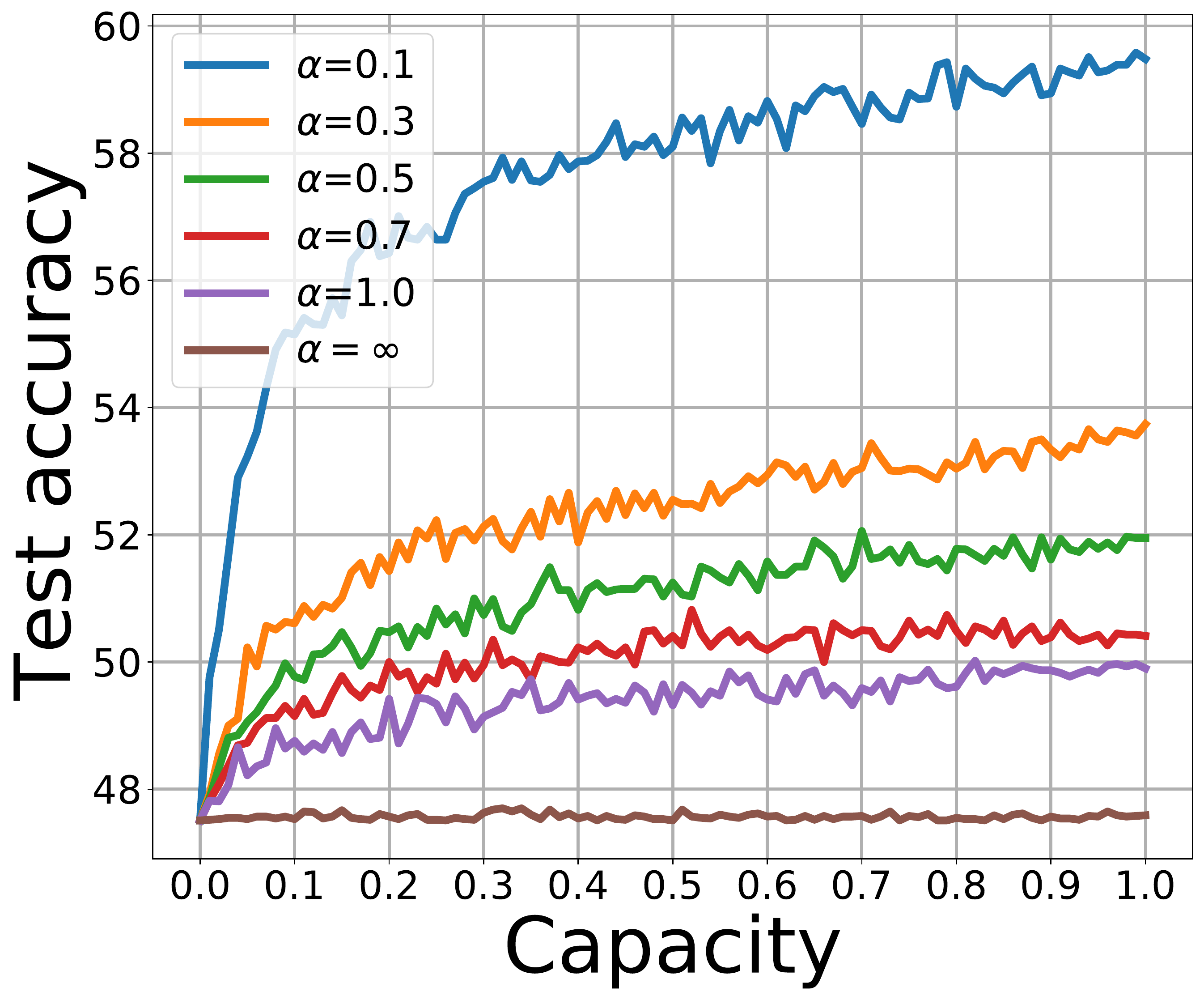}
        \subcaption[]{\small CIFAR-100}
    \end{subfigure}
    \hfill
    \caption[]
    {\small Test accuracy vs capacity (local datastore size). The capacity is normalized with respect to the initial size of the client's dataset partition. Smaller values of $\alpha$ correspond to more heterogeneous data distributions across clients.} 
    \label{f:capacity_effect}
\end{figure*}

\textbf{Training details.} In all experiments with CIFAR-10 and CIFAR-100, training spanned $200$ rounds with full clients' participation at each round for all methods. The learning rate was reduced by a factor $10$ at round $100$ and then again at round $150$. For Shakespeare, $10\%$ of clients were sampled uniformly at random without replacement at each round, and we trained for $300$ rounds with a constant learning rate.
For FEMNIST, $5\%$ of the clients participated at each round for a total $1000$ rounds, with the learning rate dropping by a factor $10$ at round $500$ and $750$. In all our experiments we employed the following aggregation scheme
\begin{equation}
    \vec{w}_{t+1} = \sum_{m \notin \mathbb{S}_{t}}\frac{n_m}{n} \vec{w}_{t} + \sum_{m\in \mathbb{S}_{t}}\frac{n_m}{n} \vec{w}_{t}^{m},
\end{equation}
where $\vec{w}_{t}$, $\vec{w}_{t}^{m}$, and $\mathbb{S}_{t}$  denote, respectively, the global model, the updated model at client $m$, and the set of clients participating to training at round $t$. 
 
In all our experiments, local hypotheses follow Eq.~\eqref{eq:gaussian_kernel} with $d(\cdot)$ being the Euclidean distance. kNN retrieval relied on FAISS library~\citep{faiss}.

\section{Experiments}
\label{sec:experiment}

\begin{figure*}[t]
    \centering
    \begin{subfigure}[b]{0.35\textwidth}  
        \centering 
        \includegraphics[width=\textwidth, keepaspectratio]{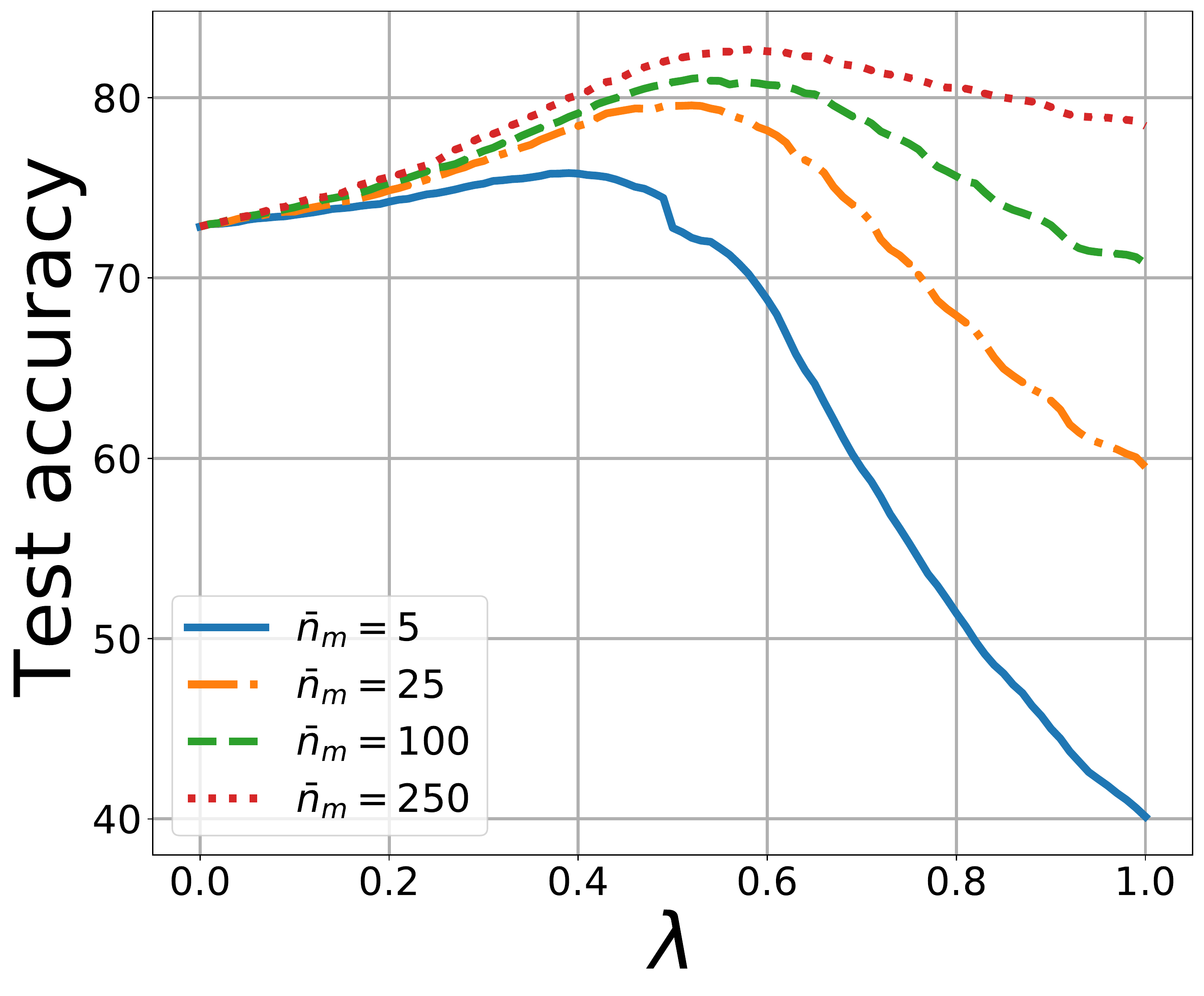}
        \subcaption[]{\small CIFAR-10}
    \end{subfigure}
    \hspace{0.1\textwidth}
    \begin{subfigure}[b]{0.35\textwidth}
        \centering
        \includegraphics[width=\textwidth, keepaspectratio]{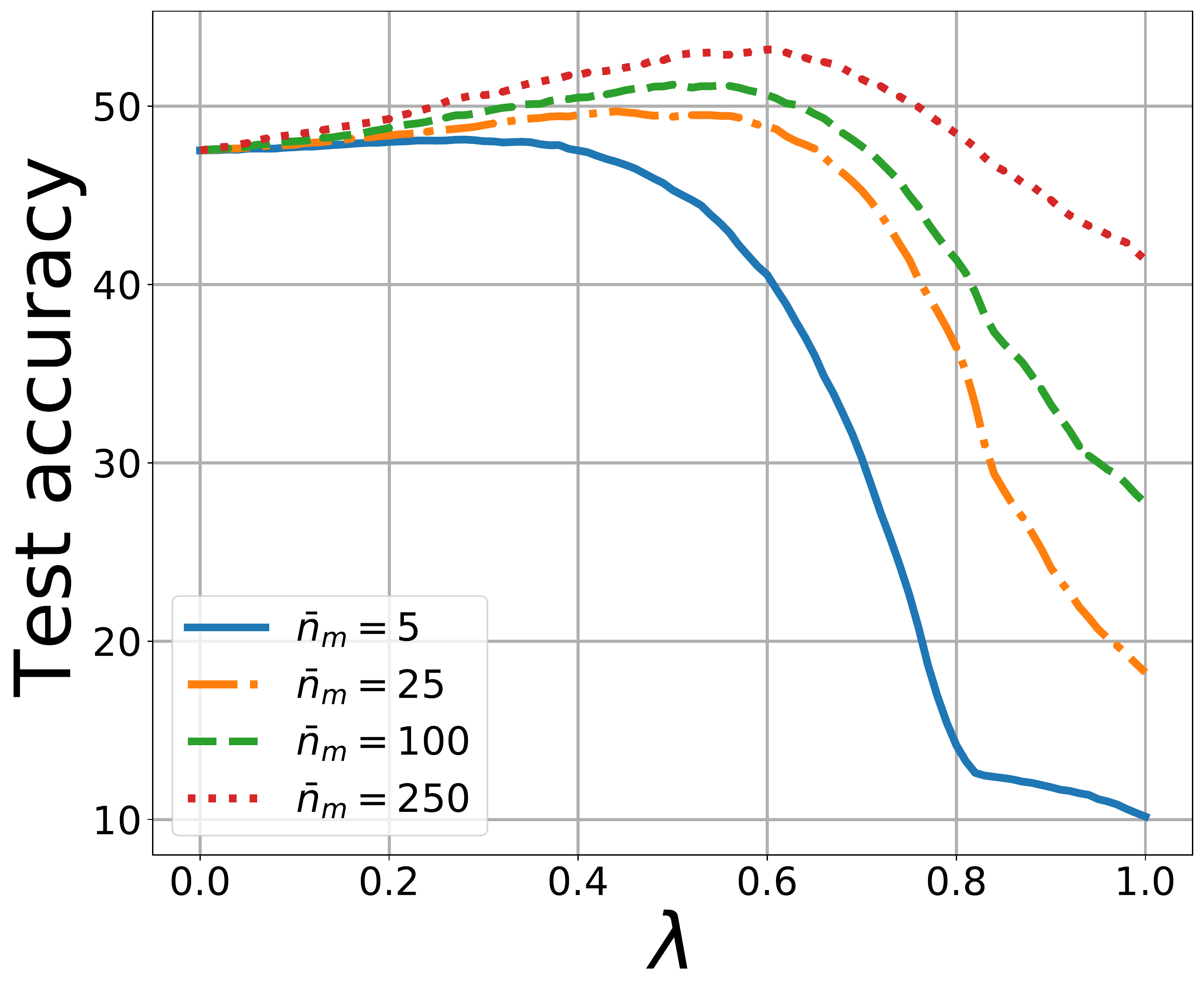}
        \subcaption[]{\small CIFAR-100}
    \end{subfigure}
    \hfill
    \caption[]
    {\small Test accuracy vs the interpolation parameter $\lambda$ (shared across clients) for different average local dataset sizes. For $\lambda=1$ (resp.~$\lambda=0$) the client uses only the kNN model (resp.~the global model).} 
    \label{f:lambda_effect}
\end{figure*}

\textbf{Average performance of personalized models.} The performance of each personalized  model (which coincides with the global one in the case of \FedAvg) is evaluated on the local test dataset (unseen at training). Table~\ref{tab:results_summary} shows the average weighted accuracy with weights proportional to local dataset sizes. \knnper{} consistently achieves the highest accuracy across all datasets. We observe that \texttt{Local} performs much worse than any other FL method as expected (e.g., 25 pp w.r.t.~\knnper{} or 22 pp w.r.t.
~to \Ditto{} on CIFAR-10). 
\texttt{Local} outperforms some other FL methods on CIFAR-10/100 (v2). This splitting was proposed in pFedGP's paper---where the same result is observed \citep[Table~1]{achituve2021personalized}. This occurs because each client only receives samples for a few classes, and then its local task is much easier than the global one.

\textbf{Fairness across clients.} Table~\ref{tab:results_summary} also shows the bottom decile of the accuracy of personalized models, i.e., the ($M/10$)-th worst accuracy (the minimum accuracy is particularly noisy, notably because some local test datasets are very small). We observe that even clients with the worst personalized models are still better off when \knnper{} is used for training. 

\textbf{Generalization to unseen clients.}  An advantage of \knnper{} is that a ``new'' client arriving after training may easily learn a personalized model: it may simply retrieve the global model (whose training it did not participate to) from the orchestrator and use it to build the local datastore for kNN.
Even if this scenario was not explicitly considered in their original papers, other personalized FL methods can also be adapted to new clients as follows. 
\FedAvg{}$+$ personalizes the global model through stochastic gradient updates on the new client's local dataset. \texttt{Ditto} operates similarly, but maintains a penalization term proportional to the distance between the personalized model and the global model. \texttt{FedRep} trains the  network head using the local dataset, while freezing the body as in the global model.
For \pFedGP{} new clients inherit the previously trained shared network and compute their local kernel.
\texttt{ClusteredFL} assigns the new client  to one learned cluster model using a held-out validation set. In the case of \FedAvg{}, there is no personalization and the new client uses directly the global model.
We performed an experiment where only $80\%$ of the clients participated to the training and the remaining $20\%$ joined later. Results in Table~\ref{tab:results_summary_unsen_clients} show that, despite its simplicity in dealing with new clients, \knnper{} still outperforms all other methods.

\textbf{Effect of local dataset's size.} Beside its relevance for some practical scenarios, the distinction between old and new clients also helps us to evaluate how different factors contribute to the final performance of \knnper. For example,  to understand how the size of the local dataset affects performance, we reduced proportionally the size of new clients' local datasets, while maintaining unchanged the global model, which was trained on old clients. \cref{f:capacity_effect} shows that new clients still reap most of \knnper's benefits   even if their local datastore is reduced by a factor 3. Note that if we had changed the local dataset sizes also for old clients, the global model (and then the representation) would have changed too, making it difficult to isolate  the effect of the local datastore size. We show the results for this experiment in \cref{f:hetero_effect_appendix} (\cref{app:experiments}).

\textbf{Effect of data heterogeneity.} 
\cref{f:capacity_effect} also shows that, as expected by Theorem~\ref{thm:generalization_bound}, the benefit of the memorization mechanism is larger when data distributions are more heterogeneous (smaller $\alpha$). While other methods also benefit from higher heterogeneity, \knnper{} appears to address statistical heterogeneity more effectively (\cref{f:hetero_effect_benchmark}).
Note that if local distributions were identical ($\alpha \to \infty$), no personalization method would provide any advantage. 

\textbf{Hyperparameters.} \knnper's performance is not highly sensitive to the value $k$ which can be selected between $7$ and~$14$ for CIFAR-10 and between $5$ and $12$ for CIFAR-100 with less than 0.2 percentage points of accuracy variation (see \cref{f:k_effect} in Appendix~\ref{app:experiments}). Similarly, scaling the Euclidean distance by a factor $\sigma$ has almost no effect for values of $\sigma$ between $0.1$ and $100$ and between $1$ and $100$, respectively for CIFAR-10 and CIFAR-100 (see \cref{f:scale_effect} in Appendix~\ref{app:experiments}). 
The interpolation parameter $\lambda_{m}$  plays a more important role. Experiments in Appendix~\ref{app:experiments} (\cref{f:lambda_opt}) show that, as expected, the larger the local dataset, the more clients rely on the local kNN model. Interestingly, clients give a larger weight to the kNN model  than to the global one ($\lambda > 1/2$)
 for datasets with just one hundred samples (\cref{f:lambda_effect}).

\textbf{Effect of global model's quality.} Assumption~\ref{assum:lipschitz} stipulates that the smaller the expected loss of the global model, the better representations' distances capture the variability of $\vec{x} \mapsto \mathcal{D}_{m}\left(
\cdot|\vec{x}\right)$ and then the more accurate the kNN model.
This effect is quantified by Lemma~\ref{lem:local_generalization_bound}, where the loss of the local memorization mechanism is upper bounded by a term that depends linearly on the loss of the global model. In order to validate this assumption, we study the relation between the test accuracies of the global model and \texttt{kNN-Per}. In particular, we trained a global model for CIFAR-10, in a centralized way, and we save the weights at different stages of the training, leading to global models with different accuracy.
\cref{f:base_model_effect_main} shows the test accuracy of \knnper{} with $\lambda=1$ (i.e., when only the kNN predictor is used) as a function of the global model's test accuracy for different levels of heterogeneity. We observe that, quite unexpectedly, the relation between the two accuracies is almost linear. Additional experiments (also including CIFAR-100) in Appendix~\ref{app:experiments} confirm these findings.

\begin{figure}
        \centering
        \includegraphics[width=0.35\textwidth]{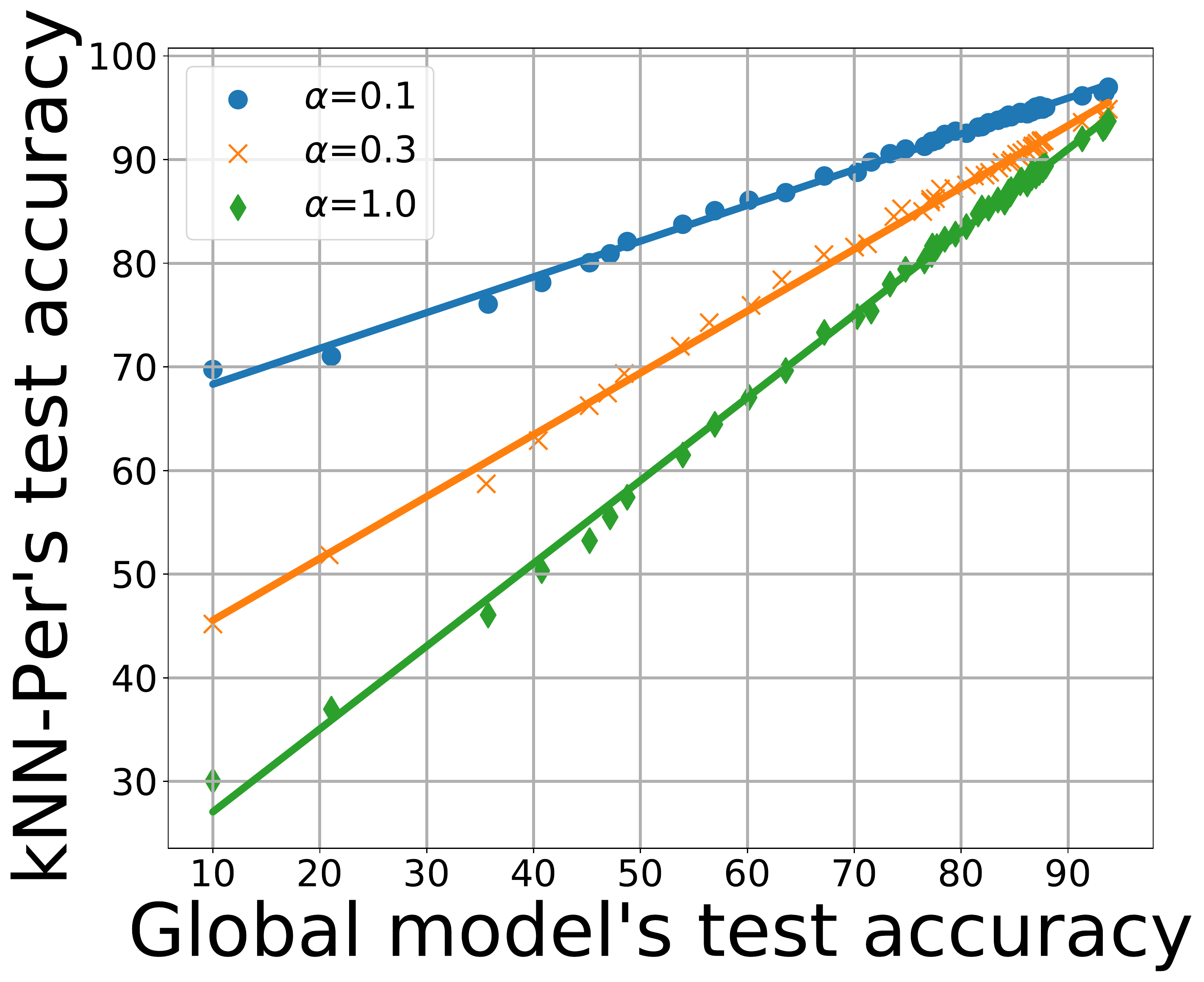}
        \caption[]{Effect of the global model's quality on the test accuracy of \texttt{kNN-Per} with $\lambda=1$. CIFAR-10.}
        \label{f:base_model_effect_main}
\end{figure}

\paragraph{Robustness to distribution shift.}  
\knnper{} offers a simple and effective way to address statistical heterogeneity in a dynamic environment where  client's data distributions change after training.
We simulate such a dynamic environment as follows.
Client $m$ initially has a datastore with instances sampled from data distribution $\mathcal{D}_m$. 
At each time step $t<t_{0}$, client $m$ receives a batch of  instances drawn from $\mathcal{D}_m$. 
At time step $t_{0}$, a data distribution shift takes place, i.e., for $ t_{0} \le t \le T$, client  $m$ receives instances drawn from a data distribution $\mathcal{D}_m'\neq \mathcal{D}_m$. 
Upon receiving new instances, client $m$ may use those instances to update its datastore. We consider $3$ different strategies: (1) \emph{first-in-first-out} (FIFO) where, at time step $t$, new instances replace the oldest ones;
(2) \emph{concatenate}, where the new samples are simply added to the datastore;
(3) \emph{fixed datastore}, where the datastore is not updated at all. 
Figure~\ref{f:stream_main} shows the evaluation of the test accuracy across time.
If  clients do not update their datastores, there is a significant drop in accuracy as soon as the distribution changes at $t_0 = 50$.
Under FIFO, we observe some random fluctuations for the accuracy for $t< t_0$, as repository changes affect the kNN predictions. While accuracy inevitably drops for $t=t_0$, it then increases as datastores are progressively populated by instances from the new distributions. 
Under the ``concatenate''  strategy, results are similar, but 1)~accuracy increases for $t < t_0$ as the quality of kNN predictors  improves for larger datastores, 2) accuracy increases also for $t > t_0$, but at a slower pace than what observed under FIFO, as samples from the old distribution are never evicted. Experiments' details and results for CIFAR-100 are in Appendix~\ref{app:experiments}.

\begin{figure}[t]
        \centering
        \includegraphics[width=0.35\textwidth]{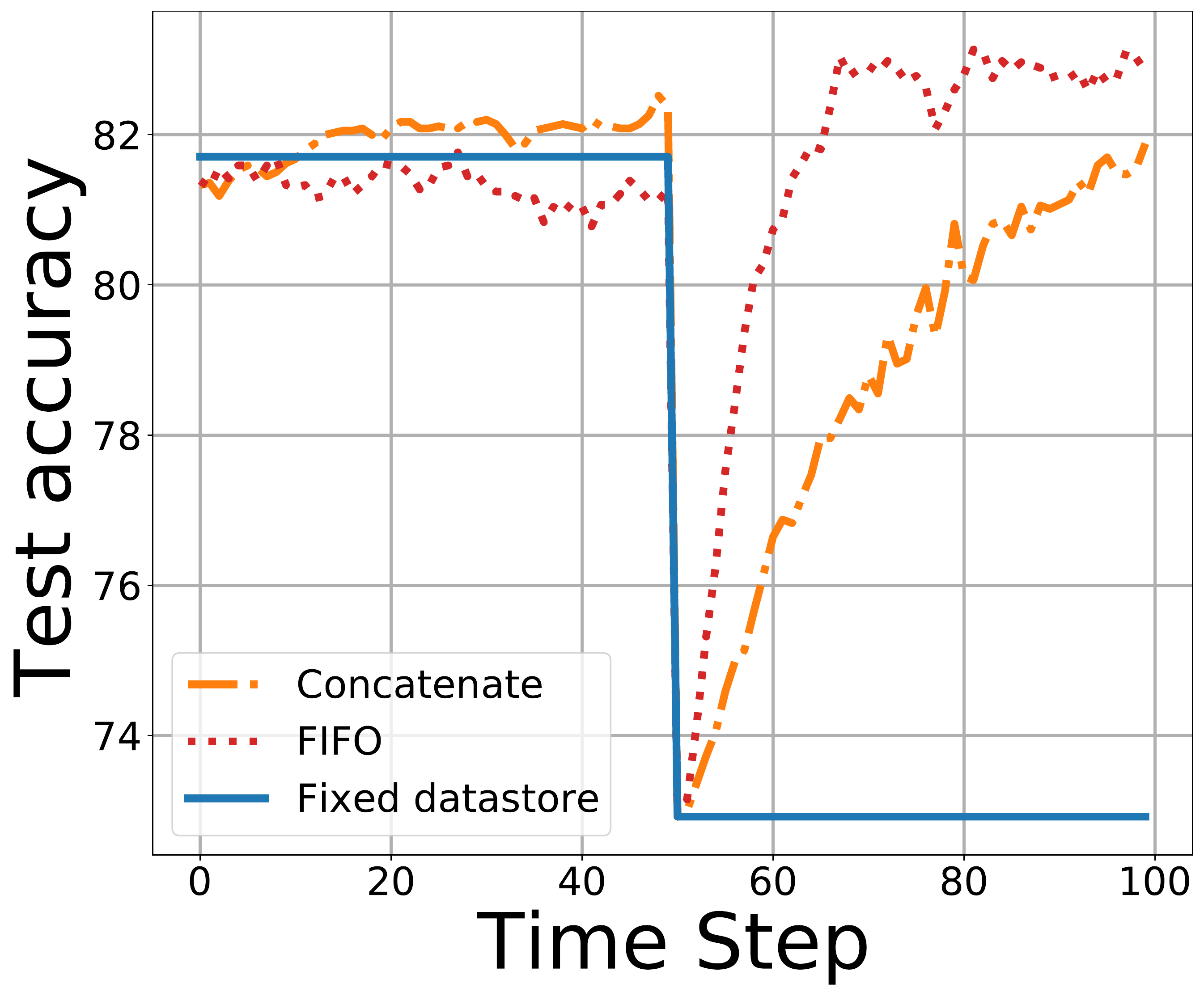}
        \caption[]{Effect of the global model's quality on the test accuracy of \texttt{kNN-Per} with $\lambda=1$. CIFAR-10.}
        \label{f:stream_main}
\end{figure}

Appendix~\ref{app:experiments} also includes experiments to evaluate  the effect of system heterogeneity and the possibility to use aggressive nearest neighbours compression techniques like \texttt{ProtoNN}~\cite{gupta17protonn}.

\section{Conclusion}
\label{sec:conclusion}
    In this paper, we showed that local memorization at each client is a simple and effective way to address statistical heterogeneity in federated learning. In particular, while a global model trained with classic FL techniques, like \FedAvg, may not deliver accurate predictions at each client, it may still provide a good representation of the input, which can be advantageously used by a local kNN model. This finding suggests that  combining memorization techniques with neural networks  has additional benefits other than those highlighted in the seminal papers~\citep{grefenstette15,joulin15} and the recent applications to natural language processing~\citep{khandelwal2019generalization,khandelwal2020nearest}.
    
    The better performance of \knnper{} in comparison to  \texttt{FedRep} and \pFedGP{} show that jointly learning the shared representation and the local models (as \texttt{FedRep} and \pFedGP{} do) may lead to potentially conflicting and interfering goals, but further study is required to  understand this interaction. Semi-parametric learning~\cite{bickel93} could be the right framework to formalize this problem, but its extension to a federated setting is still unexplored.
    

\section*{Acknowledgments}
This work has been supported by the French government, through the 3IA C\^ote d’Azur Investments in the Future project managed by the National Research Agency (ANR) with the reference number ANR-19-P3IA-0002. The authors are grateful to the OPAL infrastructure from Universit\'e C\^ote d’Azur for providing computational resources and technical support.

\bibliographystyle{unsrtnat}
\bibliography{references.bib}

\newpage

\onecolumn

\appendix

\section{Proofs}
\label{app:proofs}
In the general description of \knnper, and in our experiments, we considered that each client $m\in [M]$ uses its whole dataset $\sS_m$ both to train the base shared model $h_\sS$---and the corresponding representation function $\phi_{h_\sS}$---and to populate the local datastore. 

In the analysis, for simplicity, we deviate by this operation and consider that each local dataset $\sS_m$ is split in two disjoint parts ($\sS_m=\sS_m' \dot\bigcup  \sS_m''$), with $\sS_m'$ used to train the base model and $\sS_m''$ used to populate the local datastore. Moreover, we assume that the two parts have the same size, i.e., $n_m'= n_m''= n_m/2$ for all $m \in [M]$, where $n_m'$ and $n_m''$ denote the size of $\sS_m'$ and $\sS_m''$, respectively. In general, the result holds if the two parts have a fixed relative size across clients (i.e., $n_{m_1}'/n_{m_1} = n_{m_2}'/n_{m_2}$ for all $m_1$ and $m_2$ in $[m]$).

Let $\sS'$ denote the whole data used to train the base model, i.e., $\sS'= \bigcup_{m \in [M]} \sS'_m$.
We observe that the base model $h_\sS$ is only function of $\sS'$, and then we can write $h_{\sS'}$. Instead, the local model $h_{\sS_m}^{(1)}$ is both a function of $\sS'$ (used to learn the shared representation $\phi_\sS'$) and of $\sS_m''$ (used to populate the datastore). In order to stress such dependence, we then write $h_{\sS_m'', \sS'}^{(1)}$.

\subsection{Proof of Theorem~\ref{thm:generalization_bound}}

\begin{repthm}{thm:generalization_bound}
    Suppose that Assumptions~\ref{assum:bounded_loss}--\ref{assum:lipschitz} hold, and consider $m\in[M]$ and $\lambda_m \in (0,1)$, then there exist constants $c_{1}, c_{2}, c_{3}$, $c_{4}$, and $c_5 \in\mathbb{R}$, such that

    \begin{flalign}
        \label{eq:bound2}
        \E_{\sS\sim \otimes_{m=1}^{M}\mathcal{D}_{m}^{n_{m}}}  \left[\mathcal{L}_{\mathcal{D}_{m}} \left(h_{m, \lambda_m}\right)\right]  \leq  & \left(1 + \lambda_m\right)\cdot \mathcal{L}_{\mathcal{D}_{m}} \left({h}^{*}_m\right)  
        + c_{1}\left(1-\lambda_m\right)  \cdot \left(\disc_{\mathcal{H}}\left(\bar{\mathcal{D}}, \mathcal{D}_{m}\right) +1 \right) \nonumber \\
        & + c_{2}\lambda_m \cdot \frac{\sqrt{p}}{\sqrt[p+1]{n_{m}}} \cdot  \disc_{\mathcal{H}}\left(\bar{\mathcal{D}}, \mathcal{D}_{m}\right)   + c_{3}\left(1-\lambda_m\right) \cdot \sqrt{\frac{d_{\mathcal H}}{n}} \cdot \sqrt{c_{4} + \log\left(\frac{n}{d_{\mathcal H}}\right)} \nonumber \\
        & + c_{5} \lambda_m \cdot \sqrt{\frac{d_{\mathcal H}}{n}}  \cdot \sqrt{c_{4} + \log\left(\frac{n}{d_{\mathcal H}}\right)} \cdot \frac{\sqrt{p}}{\sqrt[p+1]{n_{m}}}, 
    \end{flalign}
    where $d_{\mathcal H}$ is the the VC dimension of the hypothesis class $\mathcal{H}$, $n=\sum_{m=1}^{M}n_{m}$, $\bar{\mathcal{D}} = \sum_{m=1}^{M} \frac{n_{m}}{n} \cdot \mathcal{D}_{m}$, $p$ is the dimension of representations, and $\disc_{H}$ is the label discrepancy associated to the hypothesis class $\mathcal{H}$.
\end{repthm}

\begin{proof}
    The idea of the proof is to bound both the expected error of the \emph{shared} base model (Lemma~\ref{lem:global_generalization_bound})  and the error of the local kNN retrieval mechanism (Lemma~\ref{lem:local_generalization_bound}) before using the convexity of the loss function to bound the error of $h_{m,\lambda_m}$.
    
    Consider $\sS \sim \otimes_{m=1}^{M}\mathcal{D}_{m}^{n_{m}}$ or, equivalently, $\sS = \sS' \cup \sS''$, where 
    $\sS' \sim \otimes_{m=1}^{M}\mathcal{D}_{m}^{n_{m}/2}$,
    and  $\sS'' = \cup_{m \in [M]} \sS_m''$ and $\sS_m''\sim \mathcal{D}_{m}^{n_{m}/2}$.  
    
    For $m\in[M]$, and $\lambda_m \in (0,1)$, we have 
    \begin{equation}
        h_{m, \lambda_m} = \lambda_m \cdot h_{\sS_{m}'',\sS'}^{(1)} + \left(1-\lambda_m\right) \cdot h_{\sS'}. 
    \end{equation}
    From Assumption~\ref{assum:convex_loss} and the linearity of the expectation, it follows 
    \begin{equation}
        \mathcal{L}_{\mathcal{D}_{m}}\left(h_{m,\lambda_m}\right) \leq \lambda_m \cdot  \mathcal{L}_{\mathcal{D}_{m}}\left( h_{\sS_{m}'',\sS'}^{(1)}\right) + \left(1-\lambda_m\right)\cdot \mathcal{L}_{\mathcal{D}_{m}}\left(h_{\sS'}\right). 
    \end{equation}
    Using Lemma~\ref{lem:local_generalization_bound} and Lemma~\ref{lem:global_generalization_bound}, and applying expectation over samples $\sS \sim \otimes_{m=1}^{M}\mathcal{D}_{m}^{n_{m}}$, we have
    \begin{flalign}
         \E_{\sS \sim \otimes_{m=1}^{M}\mathcal{D}_{m}^{n_{m}}} \left[\mathcal{L}_{\mathcal{D}_{m}} \left(h_{m, \lambda_m}\right)\right] 
         & \le \lambda_m \cdot \E_{\sS' \sim \otimes_{m=1}^{M}\mathcal{D}_{m}^{n_{m}/2}} \left[ 
            \E_{\sS'' \sim \otimes_{m=1}^{M}\mathcal{D}_{m}^{n_{m}/2}} \left[ \mathcal{L}_{\mathcal{D}_{m}}\left( h_{\sS_{m}'',\sS'}^{(1)}\right)
         \right]
         \right] \nonumber\\
         & \quad+ (1-\lambda_m)\cdot \E_{\sS' \sim \otimes_{m=1}^{M}\mathcal{D}_{m}^{n_{m}/2}} \left[ 
            \E_{\sS'' \sim \otimes_{m=1}^{M}\mathcal{D}_{m}^{n_{m}/2}} \Big[ \mathcal{L}_{\mathcal{D}_{m}}\left(h_{\sS'}\right)
         \Big]
         \right]\\
         & \le   2 \lambda_m \mathcal{L}_{\mathcal{D}_{m}}\left(h_{m}^{*}\right) 
         +  6 \lambda_m \gamma_{1} \frac{ \sqrt{p}}{\sqrt[p+1]{n_{m}}} \nonumber\\
         & \quad + 6 \lambda_m \gamma_{2} \frac{ \sqrt{p}}{\sqrt[p+1]{n_{m}}}\cdot
        \left( \E_{\sS' \sim \otimes_{m=1}^{M}\mathcal{D}_{m}^{n_{m}/2}} \left[ \mathcal{L}_{\mathcal{D}_{m}}\left(h_{\sS'} \right)\right] - \mathcal{L}_{\mathcal{D}_{m}}\left(h_{m}^{*} \right)\right) \nonumber\\
        & \quad+ (1-\lambda_m)\cdot \E_{\sS' \sim \otimes_{m=1}^{M}\mathcal{D}_{m}^{n_{m}/2}} \left[ 
             \mathcal{L}_{\mathcal{D}_{m}}\left(h_{\sS'}\right)
        \right]\\
        & \le   2 \lambda_m \mathcal{L}_{\mathcal{D}_{m}}\left(h_{m}^{*}\right) 
         +  6 \lambda_m \gamma_{1} \frac{ \sqrt{p}}{\sqrt[p+1]{n_{m}}} \nonumber\\
         & \quad + 6 \lambda_m \gamma_{2} \frac{ \sqrt{p}}{\sqrt[p+1]{n_{m}}}\cdot
        \left( \delta_{1} \cdot \sqrt{\frac{d_{\mathcal H}}{n}}\cdot \sqrt{\delta_{2} +   \log\left(\frac{n}{d_{\mathcal H}} \right) }  + 2 \cdot \disc_{\mathcal{H}}\left(\bar{\mathcal{D}}, \mathcal{D}_{m}\right)\right) \nonumber\\
        & \quad+ (1-\lambda_m)\cdot \left( \mathcal{L}_{\mathcal{D}_{m}}\left(h_{m}^{*} \right) + \delta_{1} \cdot \sqrt{\frac{d_{\mathcal H}}{n}}\cdot \sqrt{\delta_{2} +   \log\left(\frac{n}{d_{\mathcal H}} \right) }  + 2 \cdot \disc_{\mathcal{H}}\left(\bar{\mathcal{D}}, \mathcal{D}_{m}\right)\right)\\
        & =   (1+ \lambda_m) \mathcal{L}_{\mathcal{D}_{m}}\left(h_{m}^{*}\right) 
         +  6 \lambda_m \gamma_{1} \frac{ \sqrt{p}}{\sqrt[p+1]{n_{m}}} \nonumber\\
         & \quad + 
         6 \lambda_m \gamma_{2} \frac{ \sqrt{p}}{\sqrt[p+1]{n_{m}}} \delta_{1} \cdot \sqrt{\frac{d_{\mathcal H}}{n}}\cdot \sqrt{\delta_{2} +   \log\left(\frac{n}{d_{\mathcal H}} \right) }  + 12 \lambda_m \gamma_{2} \frac{ \sqrt{p}}{\sqrt[p+1]{n_{m}}} \cdot \disc_{\mathcal{H}}\left(\bar{\mathcal{D}}, \mathcal{D}_{m}\right) \nonumber\\
        & \quad+   \delta_{1} (1-\lambda_m) \cdot \sqrt{\frac{d_{\mathcal H}}{n}}\cdot \sqrt{\delta_{2} +   \log\left(\frac{n}{d_{\mathcal H}} \right) }  + 2 \cdot (1-\lambda_m) \disc_{\mathcal{H}}\left(\bar{\mathcal{D}}, \mathcal{D}_{m}\right).
    \end{flalign}
    Rearranging the terms and taking $c_{1} \triangleq 2$, $c_{2}\triangleq \max\{12 \gamma_{2},6 \gamma_1\}$, $c_{3}\triangleq \delta_{1}$, $c_{4} \triangleq \delta_{2}$ and $c_{5} \triangleq 6\gamma_{2}\delta_{1}$, the final result follows.
\end{proof}

\subsection{Intermediate Lemmas}
\label{sec:support_lem}

\begin{lem}    
    \label{lem:global_generalization_bound}
    Consider $m\in [M]$, then there exists constants $\delta_{1}, \delta_{2} \in \mathbb{R}$ such that
    \begin{equation}
        \E_{\sS' \sim \otimes_{m=1}^{M}\mathcal{D}_{m}^{n_{m}/2}} \left[\mathcal{L}_{\mathcal{D}_{m}}\left(h_{\sS'} \right)\right] \leq \mathcal{L}_{\mathcal{D}_{m}}\left(h_{m}^{*} \right) + \delta_{1} \cdot \sqrt{\frac{d_{\mathcal H}}{n}}\cdot \sqrt{\delta_{2} +   \log\left(\frac{n}{d_{\mathcal H}} \right) }  + 2 \cdot \disc_{\mathcal{H}}\left(\bar{\mathcal{D}}, \mathcal{D}_{m}\right),
    \end{equation}
    where $d$ is the VC dimension of the hypothesis class $\mathcal{H}$, $\bar{\mathcal{D}} = \sum_{m=1}^{M} \frac{n_{m}}{n} \cdot \mathcal{D}_{m}$ and $\disc_{H}$ is the label discrepancy associated to the hypothesis class $\mathcal{H}$.
\end{lem}

\begin{proof}
    We remind that the label discrepancy associated to the hypothesis class $\mathcal{H}$ for two distributions $\mathcal{D}_{1}$ and $\mathcal{D}_{2}$ over features and labels is defined as~\citep{mansour2020three}: 
    \begin{equation}
        \disc_{\mathcal{H}}\left(\mathcal{D}_{1}, \mathcal{D}_{2}\right) = \max_{h\in\mathcal{H}} \left|\mathcal{L}_{\mathcal{D}_{1}}\left(h\right) - \mathcal{L}_{\mathcal{D}_{2}}\left(h\right) \right|.
    \end{equation}
    Consider $m\in[M]$ and  $h^{*} \in \argmin_{h\in\mathcal{H}} \mathcal{L}_{\bar{\mathcal{D}}}(h)$. For $\sS' \sim \otimes_{m=1}^{M}\mathcal{D}_{m}^{n_{m}/2}$, we have
    \begin{flalign}
        \mathcal{L}_{\mathcal{D}_{m}} (h_{\sS'}) & - \mathcal{L}_{\mathcal{D}_{m}}\left(h_{m}^{*}\right)  & \nonumber 
        \\
        & = \mathcal{L}_{\mathcal{D}_{m}} (h_{\sS'}) - \mathcal{L}_{\bar{\mathcal{D}}}\left(h_{\sS'}\right) +  \mathcal{L}_{\bar{\mathcal{D}}}\left(h_{\sS'}\right) -  \mathcal{L}_{\bar{\mathcal{D}}}\left(h_{m}^{*}\right) + \mathcal{L}_{\bar{\mathcal{D}}}\left(h_{m}^{*}\right) - \mathcal{L}_{\bar{\mathcal{D}}}\left(h^{*}\right) + 
        \mathcal{L}_{\bar{\mathcal{D}}}\left(h^{*}\right) -
        \mathcal{L}_{\mathcal{D}_{m}}\left(h_{m}^{*}\right)& 
        \\
        & = \underbrace{\mathcal{L}_{\mathcal{D}_{m}} (h_{\sS'}) - \mathcal{L}_{\bar{\mathcal{D}}}\left(h_{\sS'}\right)}_{\leq \disc_{\mathcal{H}}\left(\mathcal{D}_{m}, \bar{\mathcal{D}}\right) } + 
        \underbrace{\mathcal{L}_{\bar{\mathcal{D}}}\left(h_{m}^{*}\right) - \mathcal{L}_{\mathcal{D}_{m}}\left(h_{m}^{*}\right)}_{\leq \disc_{\mathcal{H}}\left(\mathcal{D}_{m}, \bar{\mathcal{D}}\right) } +
        \underbrace{\mathcal{L}_{\bar{\mathcal{D}}}\left(h^{*}\right) - \mathcal{L}_{\bar{\mathcal{D}}}\left(h_{m}^{*}\right)}_{\leq 0} + 
        \mathcal{L}_{\bar{\mathcal{D}}}\left(h_{\sS'}\right) - \mathcal{L}_{\bar{\mathcal{D}}}\left(h^{*}\right)
        \\
        & \leq 2 \cdot \disc_{\mathcal{H}}\left(\mathcal{D}_{m}, \bar{\mathcal{D}}\right) + \mathcal{L}_{\bar{\mathcal{D}}}\left(h_{\sS'}\right) - \mathcal{L}_{\bar{\mathcal{D}}}\left(h^{*}\right) 
        \\
        & = 2 \cdot \disc_{\mathcal{H}}\left(\mathcal{D}_{m}, \bar{\mathcal{D}}\right) 
        + \mathcal{L}_{\bar{\mathcal{D}}}\left(h_{\sS'}\right) 
        - \mathcal{L}_{\sS'}\left(h_{\sS'}\right) 
        + \underbrace{\mathcal{L}_{\sS'}\left(h_{\sS'}\right) 
        - \mathcal{L}_{\sS'}\left(h^*\right)}_{\le 0} 
        + \mathcal{L}_{\sS'}\left(h^*\right) 
        - \mathcal{L}_{\bar{\mathcal{D}}}\left(h^{*}\right)
        \\
        & \leq 2 \cdot \disc_{\mathcal{H}}\left(\mathcal{D}_{m}, \bar{\mathcal{D}}\right) + 2 \cdot \sup_{h\in\mathcal{H}} \left|\mathcal{L}_{\bar{\mathcal{D}}}\left(h\right) - \mathcal{L}_{\sS'}\left(h\right)\right|. \label{eq:global_generalization_bound_part1} &
    \end{flalign}
    We now bound $\E_{\sS'\sim \otimes_{m=1}^{M}\mathcal{D}_{m}^{n_{m}/2}}\sup_{h\in\mathcal{H}} \left|\mathcal{L}_{\bar{\mathcal{D}}}\left(h\right) - \mathcal{L}_{\sS'}\left(h\right)\right|$. We first observe that for every $h\in\mathcal{H}$, we can write $\mathcal{L}_{\bar{D}}(h) = \E_{\mathcal{S'}\sim \otimes_{m=1}^{M}\mathcal{D}_{m}^{n_{m}/2}} \mathcal{L}_{\mathcal{S'}}\left(h\right)$. Therefore, despite the fact that the samples in $\sS'$ are not i.i.d., we can follow the same steps as in the proof of \citet[Theorem~6.11]{shalev2014understanding}, and conclude
    \begin{equation}
        \E_{\sS'\sim \otimes_{m=1}^{M}\mathcal{D}_{m}^{n_{m}/2}}\sup_{h\in\mathcal{H}} \left|\mathcal{L}_{\bar{\mathcal{D}}}\left(h\right) - \mathcal{L}_{\sS'}\left(h\right)\right| \leq \frac{4 + \sqrt{\log\left(\tau_{\mathcal{H}}\left(n \right)\right)}}{\sqrt{n}},
    \end{equation}
    where $\tau_{\mathcal{H}}$ is the growth function of class $\mathcal{H}$. 
    
    Let $d$ denote the VC dimension of $\mathcal H$. From Sauer's lemma \citep[Lemma~6.10]{shalev2014understanding}, we have that for $n>d+1$, $\tau_{\mathcal{H}}(n) \leq \left(\euler n/d\right)^{d_{\mathcal H}}$. Therefore, there exist constants $\delta_{1}, \delta_{2} \in \mathbb{R}$ (e.g., $\delta_1=4$, $\delta_2=\max\{4/\sqrt{d_{\mathcal H}},1\}$), such that
    \begin{equation}
        \E_{\sS'\sim \otimes_{m=1}^{M}\mathcal{D}_{m}^{n_{m}/2}}\sup_{h\in\mathcal{H}} \left|\mathcal{L}_{\bar{\mathcal{D}}}\left(h\right) - \mathcal{L}_{\sS'}\left(h\right)\right| \leq \frac{\delta_{1}}{2} \cdot \sqrt{\frac{d_{\mathcal H}}{n}}\cdot \sqrt{\delta_{2} +   \log\left(\frac{n}{d_{\mathcal H}} \right) }.
    \end{equation}
    Taking the expectation in  Eq.~\eqref{eq:global_generalization_bound_part1} and using this inequality, we have 
    \begin{equation}
            \E_{\sS' \sim \otimes_{m=1}^{M}\mathcal{D}_{m}^{n_{m}/2}} \left[\mathcal{L}_{\mathcal{D}_{m}}\left(h_{\sS'} \right)\right] \leq \mathcal{L}_{\mathcal{D}_{m}}\left(h_{m}^{*} \right) + \delta_{1} \cdot \sqrt{\frac{d_{\mathcal H}}{n}}\cdot \sqrt{\delta_{2} +   \log\left(\frac{n}{d_{\mathcal H}} \right) }  + 2 \cdot \disc_{\mathcal{H}}\left(\bar{\mathcal{D}}, \mathcal{D}_{m}\right).
    \end{equation}
\end{proof}

The following Lemma proves an upper bound on the expected error of the $1$-NN learning rule.

\begin{lem}[Adapted from {\citep[Thm~19.3]{shalev2014understanding}}]
    \label{lem:local_generalization_bound}
    Under Assumptions~\ref{assum:bounded_loss},~\ref{assum:classification_error}, and~\ref{assum:lipschitz} for all $m\in[M]$, it holds 
    \begin{equation}
        \E_{\sS''_{m} \sim \mathcal{D}_{m}^{n_{m}/2}} \left[\mathcal{L}_{\mathcal{D}_{m}}\left(h^{(1)}_{\sS_{m}'', \sS'}\right)\right] \leq 2\mathcal{L}_{\mathcal{D}_{m}}\left(h_{m}^{*}\right) + 6\Bigg\{\gamma_{1} + \gamma_{2} \cdot
        \Big[\mathcal{L}_{\mathcal{D}_{m}}\left(h_{\sS'} \right) - \mathcal{L}_{\mathcal{D}_{m}}\left(h_{m}^{*} \right)\Big] \Bigg\}\cdot \frac{ \sqrt{p}}{\sqrt[p+1]{n_{m}}}.
    \end{equation}
\end{lem}
\begin{proof}
    Recall that for $m\in[M]$, the Bayes optimal rule, i.e., the hypothesis that minimizes $\mathcal{L}_{\mathcal{D}_{m}}(h)$ over all functions, is
    \begin{equation}
        h_{m}^{*}\left(\vec{x}\right) = \mathds{1}_{\left\{\eta_{m}\left(\vec{x}\right) > 1/2\right\}}.
    \end{equation}
We note that the $1$-NN rule can be expressed as follows:
\begin{equation}
    \label{eq:1_nn_rule}
    \left[ h^{(1)}_{\sS_m'',\sS'}\left(\vec{x}\right)\right]_y = \mathds{1}_{\left\{y=\pi^{(1)}_{\sS''_{m}}\left(\vec{x}\right)\right\}},
\end{equation}
where we are putting in evidence that the permutation $\pi_m$ depends on the dataset $\sS''_m$.
Then, under Assumption~\ref{assum:classification_error},  the loss function $l(\cdot)$ reduces to the $0$-$1$ loss.

    Consider samples $\sS \sim \otimes_{m=1}^{M}\mathcal{D}_{m}^{n_{m}}$. Using  Assumptions~\ref{assum:bounded_loss}, \ref{assum:classification_error} and~\ref{assum:lipschitz}, and following the same steps as in \citep[Lemma~19.1]{shalev2014understanding}, we have
    \begin{flalign}
        \E_{\sS''_{m} \sim \mathcal{D}_{m}^{n_{m}/2}}  & \left[\mathcal{L}_{\mathcal{D}_{m}}\left(h^{(1)}_{\sS''_{m}, \sS'}\right)\right]  -  2\mathcal{L}_{\mathcal{D}_{m}}\left(h_{m}^{*}\right) \leq   & \nonumber 
        \\
        & \Bigg\{\gamma_{1} + \gamma_{2} \cdot \Big[\mathcal{L}_{\mathcal{D}_{m}}\left(h_{\sS'} \right) - \mathcal{L}_{\mathcal{D}_{m}}\left(h_{m}^{*} \right)\Big] \Bigg\} \times \underbrace{\E_{\sS''_{m,\mathcal{X}} \sim \mathcal{D}_{m, \mathcal{X}}^{n_{m}/2},~\vec{x}\sim \mathcal{D}_{m,\mathcal{X}}}\bigg[ d \left(\phi_{h_{\sS'}}\left(\vec{x}\right) , \phi_{h_{\sS'}}\left(\pi^{(1)}_{\sS''_{m}}\left(\vec{x}\right)\right) \right) \bigg]}_{\triangleq \mathcal{T}_{\sS'}},
    \end{flalign}
    where $\sS_{m,\mathcal{X}}''$ denotes the set of input features in the dataset $\sS_m''$ and   $\mathcal{D}_{m,\mathcal{X}}$ the marginal distribution of $\mathcal{D}_{m}$ over $\mathcal{X}$. Note  that  $\sS''_{m}$ is independent from $\sS'$. 

    As in the proof of \citep[Theorem~19.3]{shalev2014understanding},
    let $T$ be an integer to be precised later on. We consider $r=T^{p}$ and $C_{1},\dots,C_{r}$ to be the cover of the set $[0,1]^{p}$ using boxes with side $1/T$. We bound the term $\mathcal{T}_{\sS'}$ independently from $\sS'$ as follows
    \begin{equation}
        \E_{\sS''_{m} \sim \mathcal{D}_{m, \mathcal{X}}^{n_{m}/2},~\vec{x}\sim \mathcal{D}_{m,\mathcal{X}}}\bigg[ d \left(\phi_{h_{\sS'}}\left(\vec{x}\right) , \phi_{h_{\sS'}}\left(\pi^{(1)}_{\sS''_{m}}\left(\vec{x}\right)\right) \right) \bigg] \leq \sqrt{p} \left(\frac{2 T^{p}}{n_{m}\euler} + \frac{1}{T}\right).
    \end{equation}
    If we set $\epsilon = 2 \left(\frac{2}{n_m}\right)^{\frac{1}{p+1}}$ and  $T = \ceil{1/\epsilon}$, it follows $1/\epsilon \le T < 2/\epsilon$ and then   
    \begin{align}
        \E_{\sS''_{m} \sim \mathcal{D}_{m, \mathcal{X}}^{n_{m}/2},~\vec{x}\sim \mathcal{D}_{m,\mathcal{X}}}\bigg[ d \left(\phi_{h_{\sS'}}\left(\vec{x}\right) , \phi_{h_{\sS'}}\left(\pi^{(1)}_{\sS''_{m}}\left(\vec{x}\right)\right) \right) \bigg] & \leq \sqrt{p} \left(\frac{2 (2/\epsilon)^{p}}{n_{m}\euler } + \epsilon\right) \\
        & = \sqrt{p} \left(\frac{1}{e} + 2\right) \left(\frac{2}{n_m}\right)^{\frac{1}{p+1}}\\
        & \le 6 \frac{\sqrt{p}}{\sqrt[p+1]{n_m}}.
    \end{align}
    Thus,
    \begin{equation}
        \E_{\sS'_{m} \sim \mathcal{D}_{m}^{n_{m}}}   \left[\mathcal{L}_{\mathcal{D}_{m}}\left(h^{(1)}_{\sS''_{m}, \sS'}\right)\right]  \leq  2\mathcal{L}_{\mathcal{D}_{m}}\left(h_{m}^{*}\right)   + 6 \frac{\sqrt{p}}{\sqrt[p+1]{n_{m}}} \Bigg\{\gamma_{1} + \gamma_{2} \cdot \Big[\mathcal{L}_{\mathcal{D}_{m}}\left(h_{\sS} \right) - \mathcal{L}_{\mathcal{D}_{m}}\left(h_{m}^{*} \right)\Big] \Bigg\}.
    \end{equation}
\end{proof}

\newpage
\section{Additional Experiments}
\label{app:experiments}
\begin{figure*}[t]
    \centering
    \begin{subfigure}[b]{0.48\textwidth}  
        \centering 
        \includegraphics[width=\textwidth, height=0.8\textwidth]{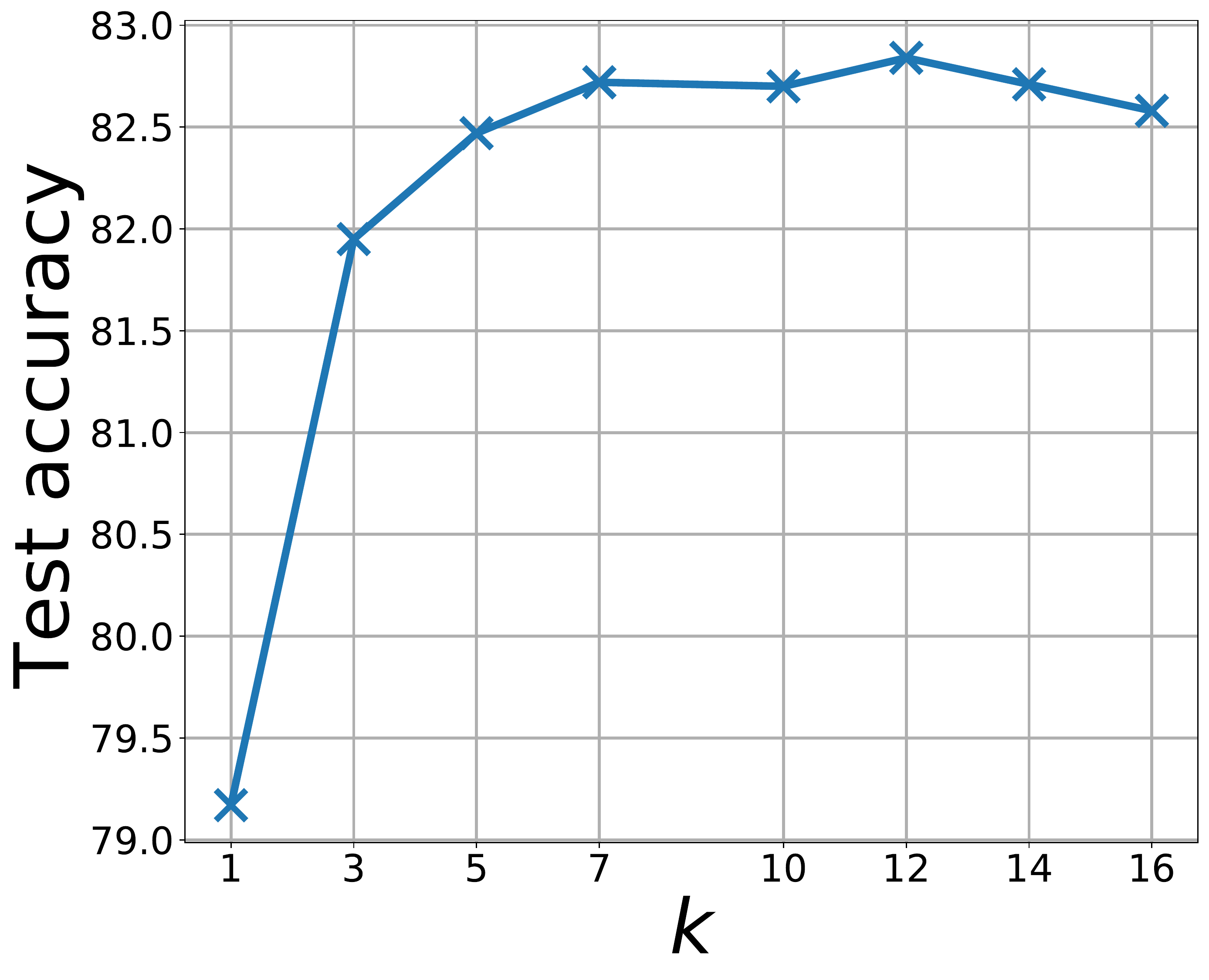}
        \subcaption[]{\small CIFAR-10.}
    \end{subfigure}
    \hfill
    \begin{subfigure}[b]{0.48\textwidth}
        \centering
        \includegraphics[width=\textwidth, height=0.8\textwidth]{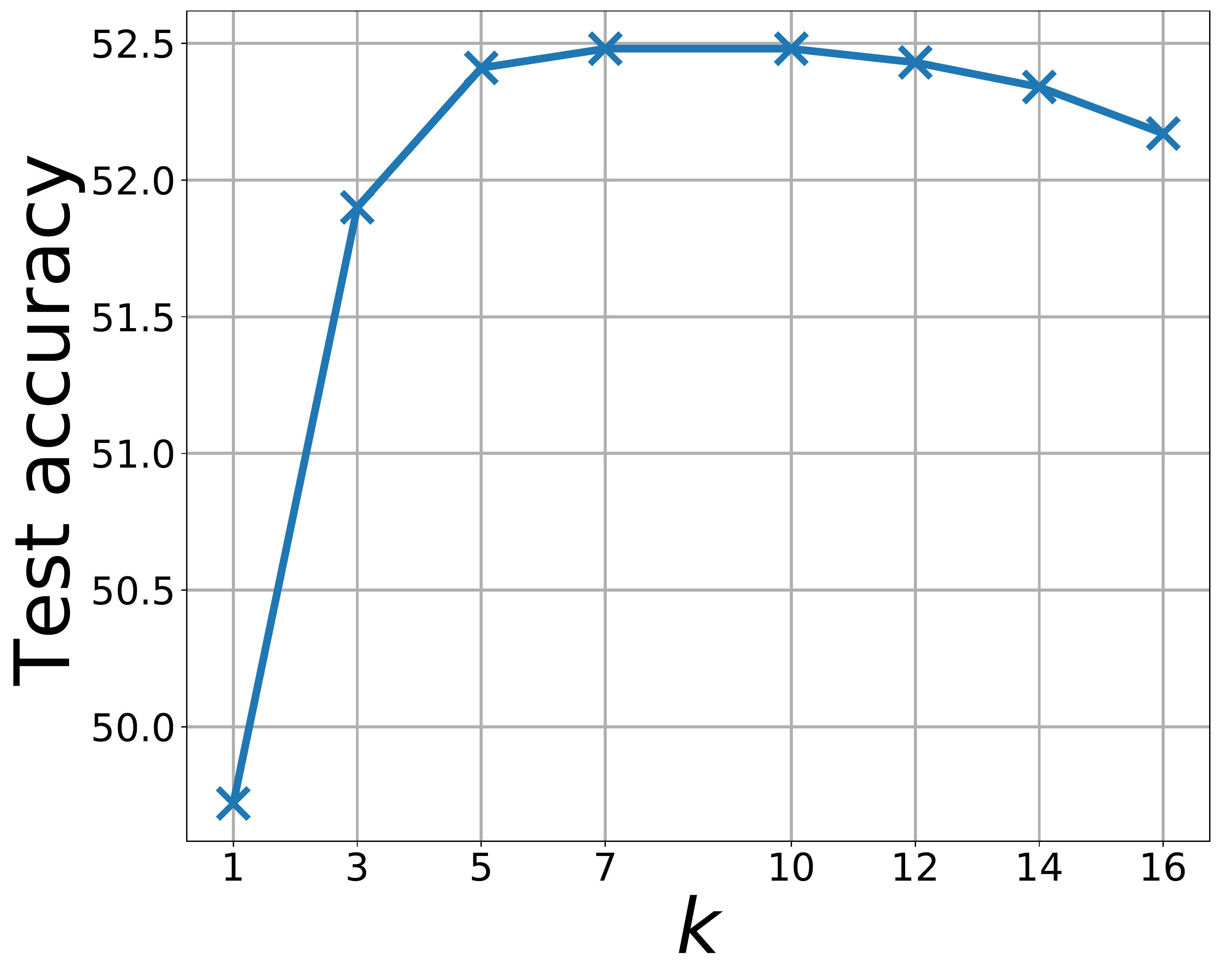}
        \subcaption[]{\small CIFAR-100.}
    \end{subfigure}
    \hfill
    \caption[]
    {\small Test accuracy vs number of neighbors $k$.} 
    \label{f:k_effect}
\end{figure*}

\begin{figure*}[!t]
    \centering
    \begin{subfigure}[b]{0.48\textwidth}  
        \centering 
        \includegraphics[width=\textwidth, height=0.8\textwidth]{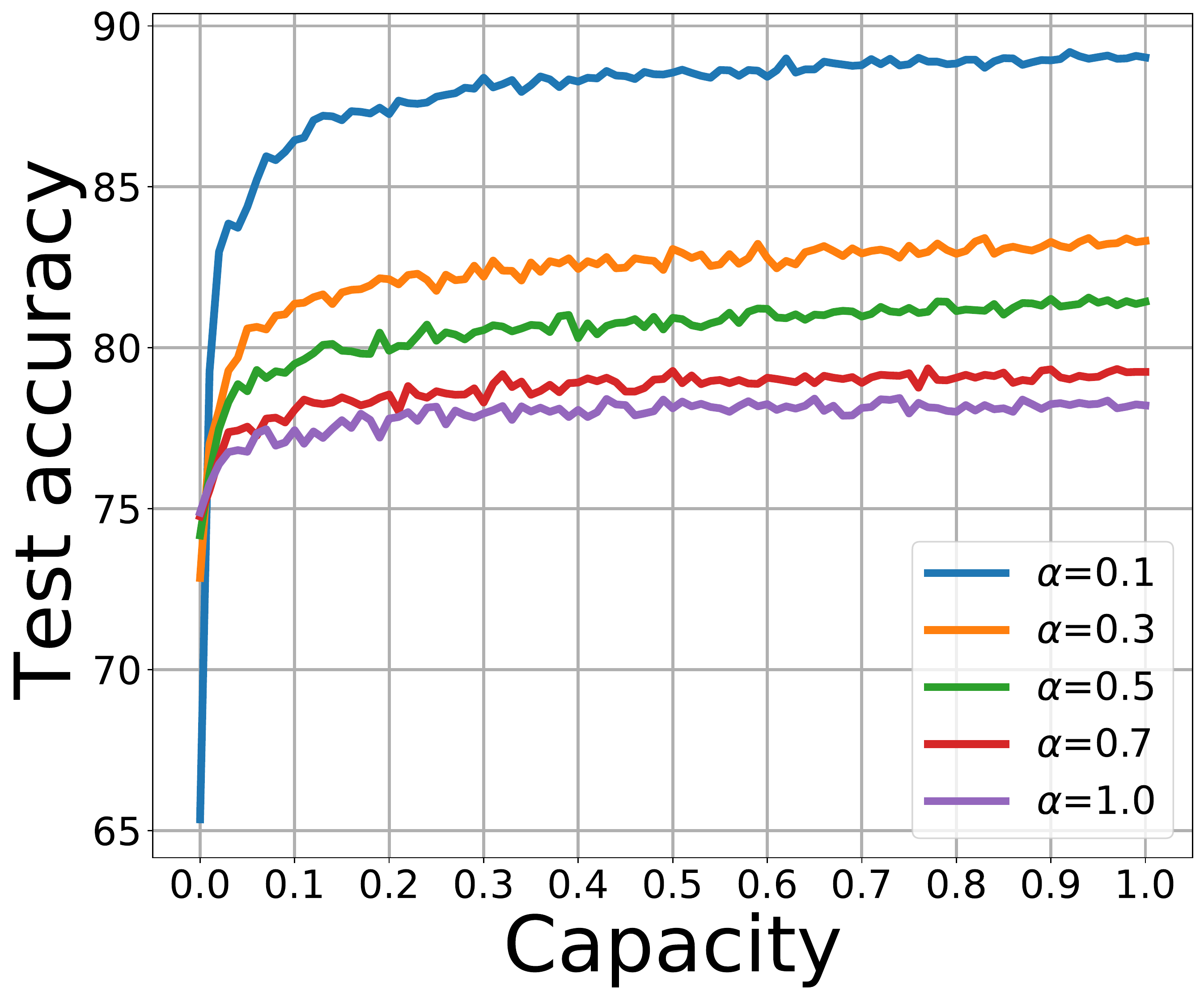}
        \subcaption[]{\small CIFAR-10.}
    \end{subfigure}
    \hfill
    \begin{subfigure}[b]{0.48\textwidth}
        \centering
        \includegraphics[width=\textwidth, height=0.8\textwidth]{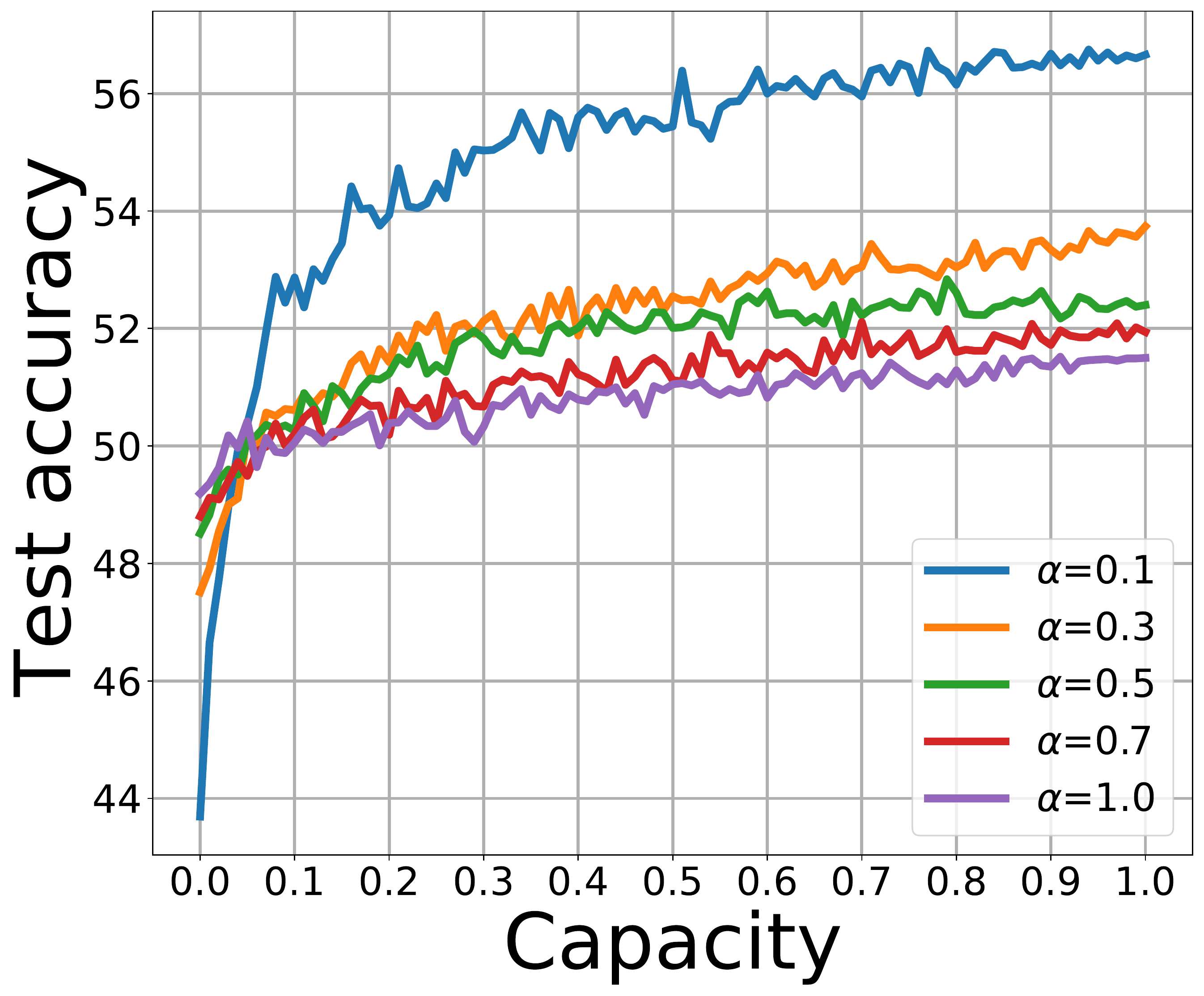}
        \subcaption[]{\small CIFAR-100.}
    \end{subfigure}
    \hfill
    \caption[]
    {\small Test accuracy vs capacity (local datastore size) when the global model is retrained for each value of $\alpha$. The capacity is normalized with respect to the initial size of the client's dataset partition. Smaller values of $\alpha$ correspond to more heterogeneous data distributions across clients. The curves start from different accuracy values for zero capacity, but are qualitatively similar to those in \cref{f:capacity_effect} for large capacities. As expected, the global model performs worse the more heterogeneous the local distributions are, but the local model is able to compensate such effect (at least partially) as far as the datastore is large enough.} 
    \label{f:hetero_effect_appendix}
\end{figure*}

\paragraph{Effect of kernel scale parameter $\sigma$.} We  consider distance metrics of the form 
\begin{equation}
    \forall \vec{z},\vec{z}' \in \mathbb{R}^{p};~d_{\sigma}\left(\vec{z}, \vec{z}'\right) = \frac{\left\|\vec{z}-\vec{z'}\right\|_{2}}{\sigma}, 
\end{equation}
where $\sigma\in\mathbb{R}^{+}$ is a  scale parameter.  \cref{f:scale_effect} shows that \knnper's performance is not highly sensitive to the selection of the length scale parameter, as scaling the Euclidean distance by a constant factor $\sigma$ has almost no effect for values of $\sigma$ between $0.1$ and $1000$. 

\begin{figure*}[t]
    \centering
    \begin{subfigure}[b]{0.48\textwidth}  
        \centering 
        \includegraphics[width=\textwidth, height=0.8\textwidth]{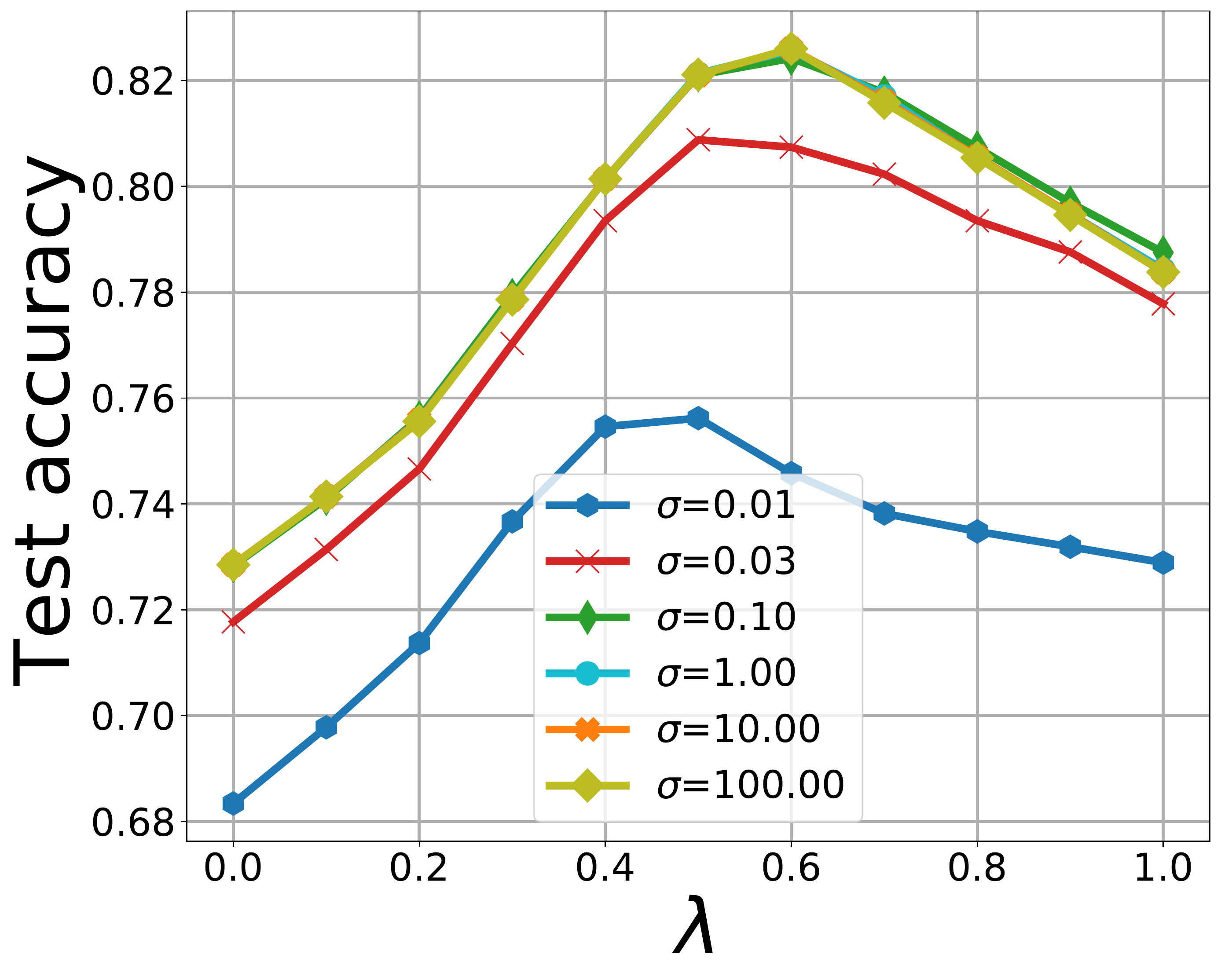}
        \subcaption[]{\small CIFAR-10.}
    \end{subfigure}
    \hfill
    \begin{subfigure}[b]{0.48\textwidth}
        \centering
        \includegraphics[width=\textwidth, height=0.8\textwidth]{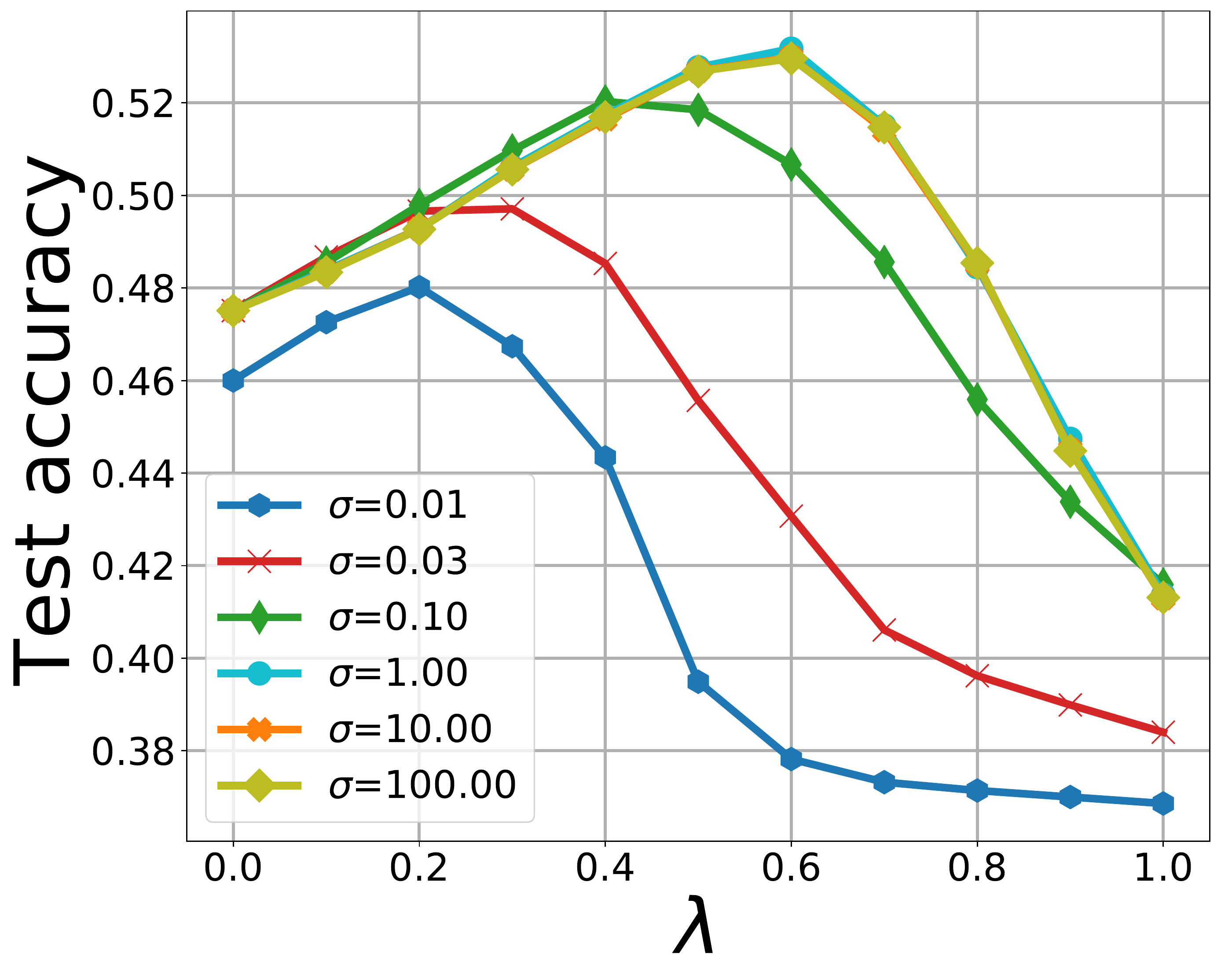}
        \subcaption[]{\small CIFAR-100.}
    \end{subfigure}
    \hfill
    \caption[]
    {\small Test accuracy vs the interpolation parameter $\lambda$ for different values of the kernel scale parameter $\sigma$.} 
    \label{f:scale_effect}
\end{figure*}

\paragraph{Effect of datastore's size on the optimal $\lambda$.}
Figure~\ref{f:lambda_opt} shows the effect of the local number of samples $n_{m}$ on the optimal mixing parameter $\lambda_{\text{opt}}$ (evaluated on the client's test dataset). The number of samples changes across clients and, for the same client, with different values of the capacity.
The figure shows a positive correlation between the local number of samples and the optimal mixing parameter and then validates the intuition that clients with more samples tend to rely more on the memorization mechanism than on the base model, as captured by the generalization bound from Theorem~\ref{thm:generalization_bound}.   

\paragraph{Effect of hardware heterogeneity.}
In our experiments above, clients' local datasets had different size, which can also be due to different memory capabilities. In order to investigate more in depth the effect of system heterogeneity, we split the new clients in two groups: ``weak'' clients with normalized capacity $1/2-\Delta C$ and ``strong'' clients with normalized capacity $1/2+\Delta C$, where $\Delta C \in (0,1/2)$ is a parameter controlling the hardware heterogeneity of the system. Note that the total amount of memory in the system is constant, but varying $\Delta C$ changes its distribution across clients from a homogeneous scenario ($\Delta C =0$) to an extremely heterogeneous one ($\Delta C= 0.5$). Figure~\ref{fig:harware_hetero} shows the effect of the hardware heterogeneity, as captured by $\Delta C$. As the marginal improvement from additional memory is decreasing (see, e.g., Fig.~\ref{f:capacity_effect}) the gain for strong clients does not compensate the loss for weak ones. The overall effect is then that the average test accuracy decreases as system heterogeneity increases. 

\paragraph{Adding compression techniques.} \knnper{} can be combined with nearest neighbours compression techniques as \texttt{ProtoNN} \citep{gupta17protonn}. 
\texttt{ProtoNN} reduces the amount of memory required by jointly learning 1) a small number of prototypes to represent the entire training set and 2) a data projection into a low dimensional space. 
We combined \knnper{} and \texttt{ProtoNN} and explored both the effect of the number of prototypes and  the projection dimension used in \texttt{ProtoNN}. For each client, the number of prototypes is set to a given fraction of the total number of available samples. We refer to this quantity also as capacity. We varied the capacity in the grid $\{i\times 10^{-1}, i \in [10]\}$, and the projection dimension in the grid $\{i \times 100, i \in [12]\} \cup \{1280\}$. Note that smaller projection dimension and less prototypes correspond to a smaller memory footprint, suited for more restricted hardware. Our implementation is based on \texttt{ProtoNN}'s official.\footnote{
    \url{https://github.com/Microsoft/EdgeML}.
}
Figure~\ref{f:heatmap_cifar10} shows that, on CIFAR-10, \texttt{ProtoNN} allows to reduce the \knnper's memory footprint by a factor four (using $n_m/3$ prototypes and projection dimension $1000$) at the cost of a limited reduction in test accuracy ($82.3\%$ versus $83.0\%$ in Table~\ref{tab:results_summary}). Note that \knnper{} with \texttt{ProtoNN} still outperforms all other methods.
On CIFAR-100, 
\texttt{ProtoNN}'s compression techniques appear less advantageous: the approach loses about 3 percentage points ($52.1\%$ versus $55.0\%$ in Table~\ref{tab:results_summary}) while only reducing memory requirement by $20\%$.

\begin{figure*}
    \centering
    \begin{subfigure}[b]{0.48\textwidth}  
        \centering 
        \includegraphics[width=\textwidth, height=0.8\textwidth]{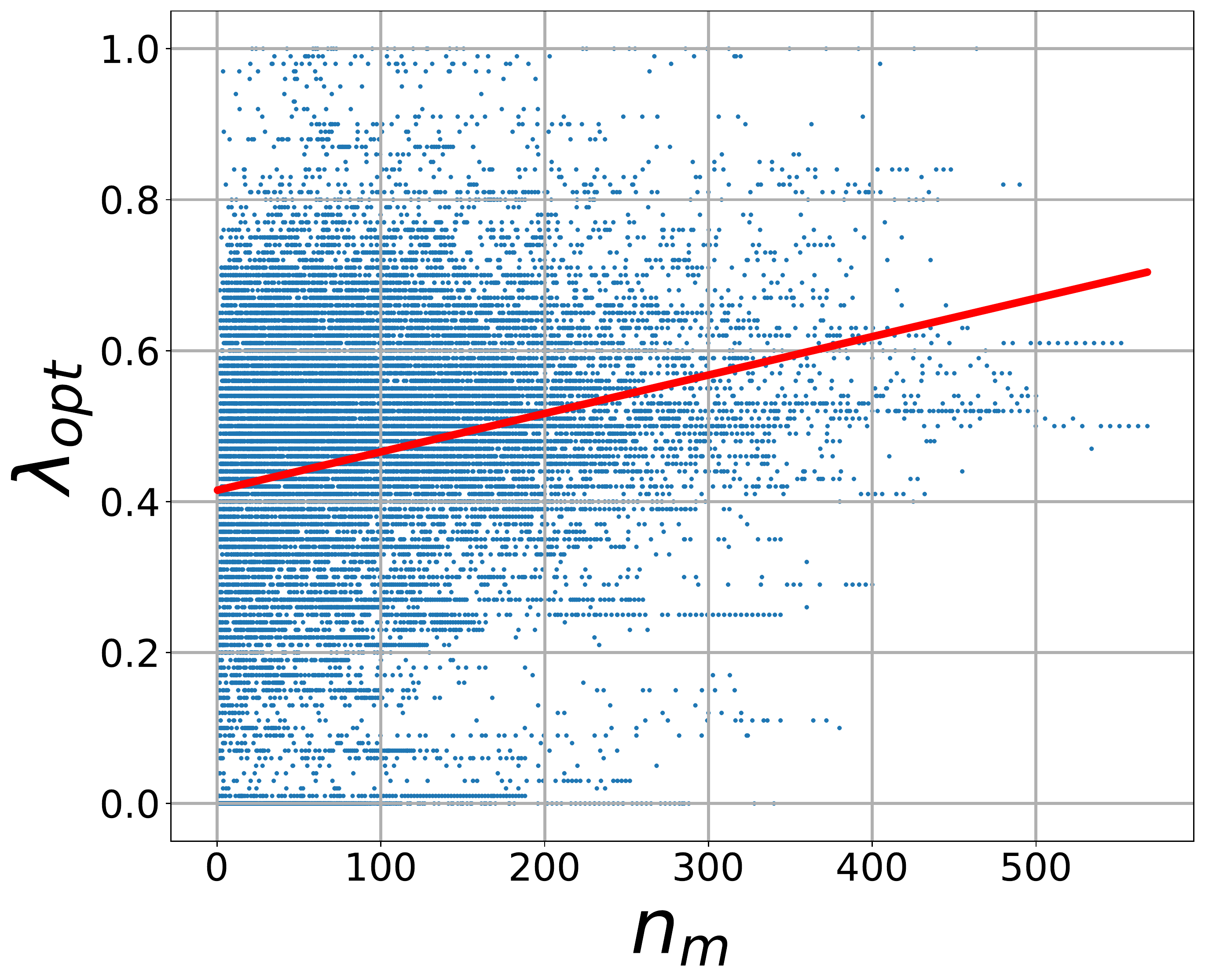}
        \subcaption[]{\small CIFAR-10.}
    \end{subfigure}
    \hfill
    \begin{subfigure}[b]{0.48\textwidth}
        \centering
        \includegraphics[width=\textwidth, height=0.8\textwidth]{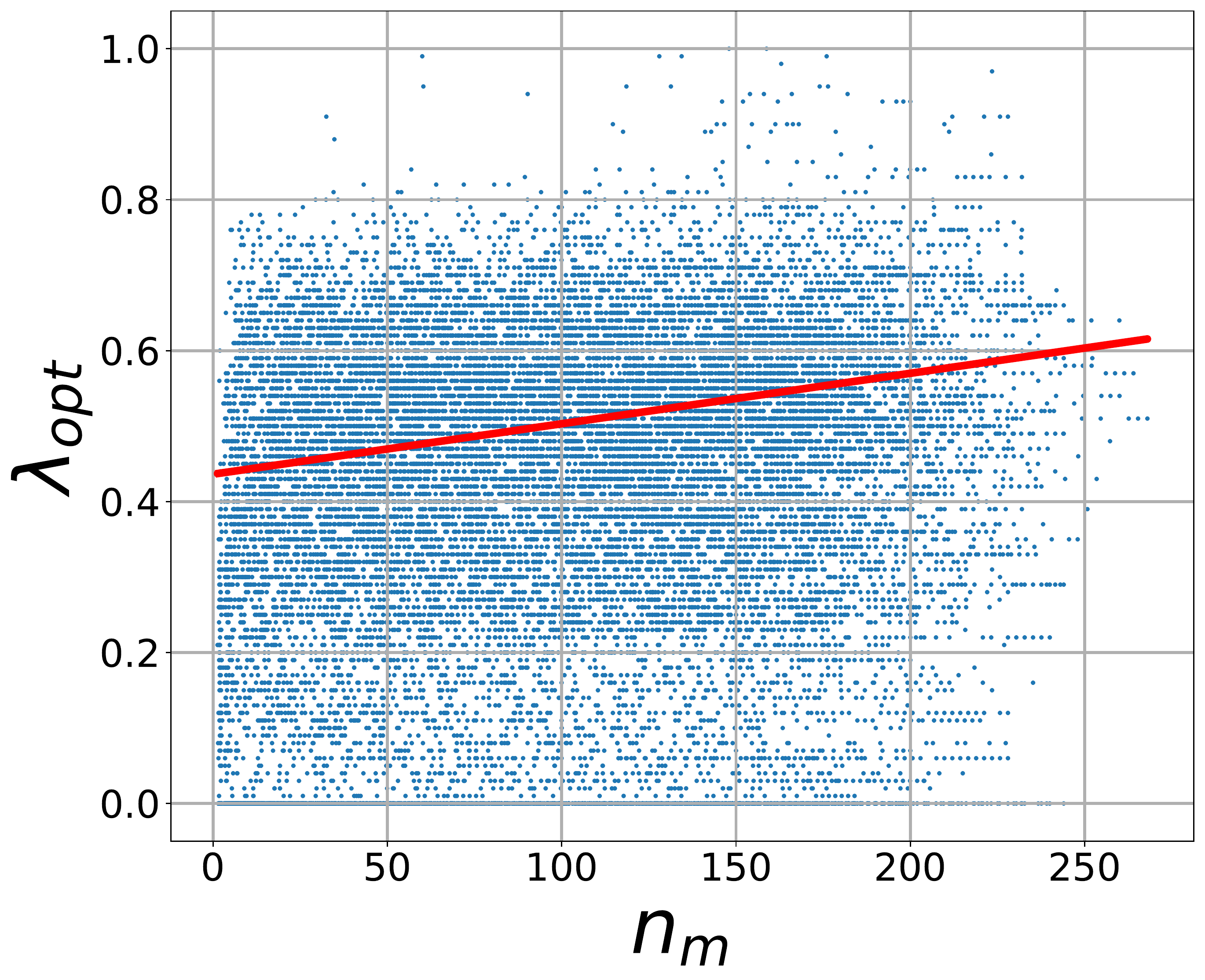}
        \subcaption[]{\small CIFAR-100.}
    \end{subfigure}
    \hfill
    \caption[]
    {\small$\lambda_{\text{opt}}$ vs local number of samples $n_{m}$.} 
    \label{f:lambda_opt}
\end{figure*}

\paragraph{Effect of global model's quality.} Assumption~\ref{assum:lipschitz} stipulates that the smaller the expected loss of the global model, the more accurate the corresponding representation. As  representation quality improves, we can expect that  kNN accuracy improves too. This effect is quantified by Lemma~\ref{lem:local_generalization_bound}, where the loss of the local memorization mechanism is upper bounded by a term that depends linearly on the loss of the global model. In order to validate this assumption, we study the relation between the test accuracies of the global model and \texttt{kNN-Per}. In particular, we train two global models, one for CIFAR-10 and the other for CIFAR-100, in a centralized way, and we save the weights at different stages of the training, leading to global models with different qualities.
Figure~\ref{f:base_model_effect} shows the test accuracy of \knnper{} with $\lambda=1$ (i.e., when only the knn predictor is used) as a function of the global model's test accuracy for different levels of heterogeneity on CIFAR-10 and CIFAR-100 datasets. We observe that, quite unexpectedly, the relation between the two accuracies is almost linear. The experiments also confirm what observed in Fig.~\ref{f:capacity_effect}: \knnper performs better when local distributions are more heterogeneous (smaller $\alpha$).
Similar plots with $\lambda$ optimized locally at every client are shown in Fig.~\ref{f:base_model_effect_lambda_opt}.

\paragraph{Robustness to distribution shift.} As previously mentioned, \knnper{} offers a simple and effective way to address statistical heterogeneity in a dynamic environment where  client's data distributions change after training.
We simulate such a dynamic environment as follows.
Client $m$ initially has a datastore built using instances sampled from a data distribution $\mathcal{D}_m$. 
For time step $t<t_{0}$, client $m$ receives a batch of $n_{m}^{(t)}$ instances drawn from $\mathcal{D}_m$. 
At time step $t_{0}$, we suppose that a data distribution shift takes place, i.e., for $ t_{0} \le t \le T$, client  $m$ receives $n_{m}^{(t)}$ instances drawn from a data distribution $\mathcal{D}_m'\neq \mathcal{D}_m$. 
Upon receiving new instances, client $m$ may use those instances to update its datastore. We consider $3$ different strategies: (1) \emph{first-in-first-out} (FIFO) where, at time step $t$, the $n_{m}^{(t)}$  oldest samples are replaced by the newly obtained samples;
(2) \emph{concatenate}, where the new samples are simply added to the datastore;
(3) \emph{fixed datastore}, where the datastore is not updated at all. 
In our simulations, we consider CIFAR-10/100 datasets with $M=100$ clients.
Once again, we used a symmetric Dirichlet distribution to generate two datasets for every client. 
In particular, for each label $y$ we sampled two vectors $p_{y}$ and $p_{y}'$ from a Dirichlet distribution of order $M=100$ and parameter $\alpha=0.3$. 
Then, for client $m$, we generated two datasets $\sS_{m}$ and $\sS'_{m}$ by allocating $p_{y, m}$ and $p_{y,m}'$ fraction of all training instances of class $y$.\footnote{
    We always make sure that $|\sS_{m}| \leq |\sS'_{m}|$.
}
Both $\sS_{m}$ and $\sS'_{m}$ are partitioned into training and test sets following the original CIFAR training/test data split. 
Half of the training set obtained from $\sS_{m}$ is stored in the datastore, while the rest is further partitioned into $t_{0}$ batches $\sS_{m}^{(0)}, \dots, \sS_{m}^{(t_{0}-1)}$.
These batches are the new samples arriving at client $m$. Similarly, $\sS_{m}'$ is partitioned into $T-t_{0}$ equally sized batches. Figure~\ref{f:stream} shows the evaluation of the test accuracy across time.
If  clients do not update their datastores, there is a significant drop in accuracy as soon as the distribution changes at $t_0 = 50$.
If datastore are updated using FIFO, we observe some random fluctuations for the accuracy for $t< t_0$, as repository changes affect the kNN predictions. While accuracy inevitably drops for $t=t_0$, it then increases as datastores are progressively populated by instances from the new distributions. 
Once all samples from the previous distributions are evicted, the accuracy settles around a new value (higher or lower than the one for $t < t_0$ depending on the difference between the new and the old distributions).
If  clients keep adding new samples to their datastores  (the ``concatenate''  strategy), results are similar, but 1)~accuracy increases for $t < t_0$ as the quality of kNN predictors  improves for larger datastores, 2) accuracy increases also for $t > t_0$, but at a slower pace than what observed under FIFO, as samples from the old distribution are never evicted.

\begin{figure}
    \centering
    \begin{minipage}{.45\textwidth}
        \centering
        \includegraphics[width=.9\linewidth]{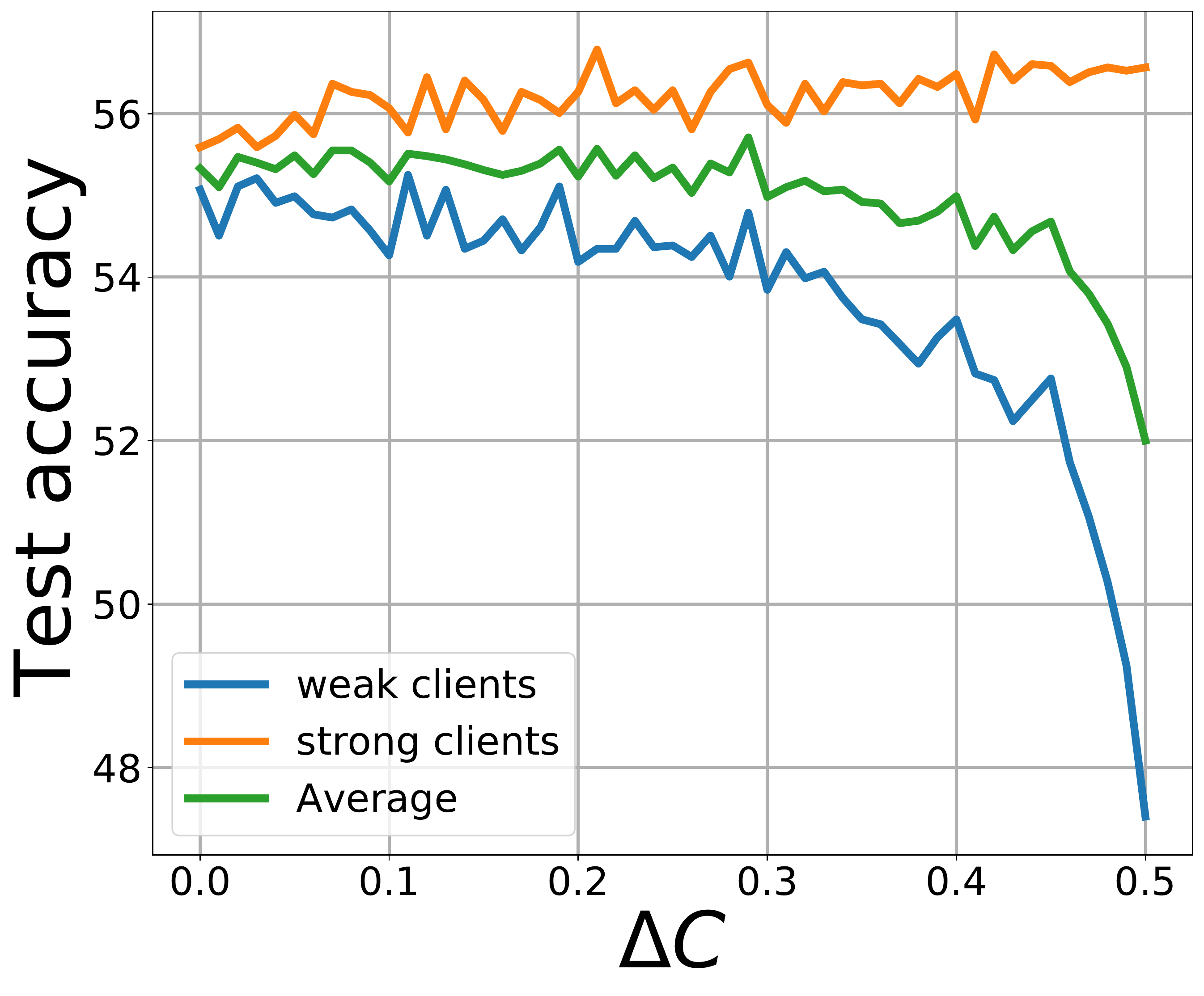}
        \caption{\small Effect of system heterogeneity across clients on CIFAR-100 dataset. The size of the local datastore increases (resp.~decreases) with $\Delta C$ for strong (resp.~weak) clients.}
        \label{fig:harware_hetero}
    \end{minipage}%
    \hfill
    \begin{minipage}{.45\textwidth}
        \centering
        \includegraphics[width=.9\linewidth]{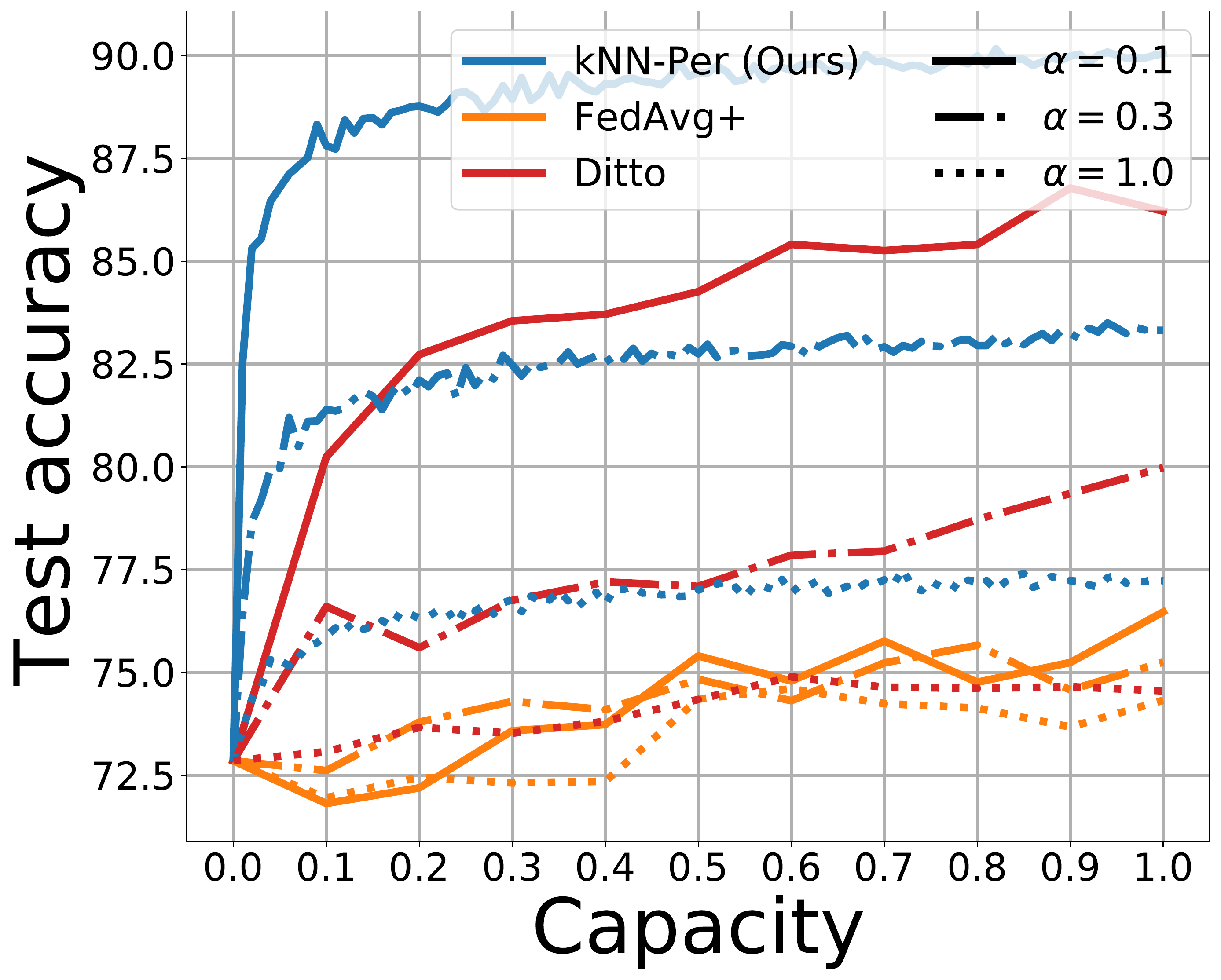}
        \caption{\small Test accuracy vs capacity (local datastore size) for different methods on CIFAR-10. The capacity is normalized with respect to the initial size of the client's dataset partition.}
        \label{f:hetero_effect_benchmark}
    \end{minipage}
\end{figure}

\begin{figure*}[!t]
    \centering
    \begin{subfigure}[b]{0.48\textwidth}  
        \centering 
        \includegraphics[width=\textwidth, height=0.8\textwidth]{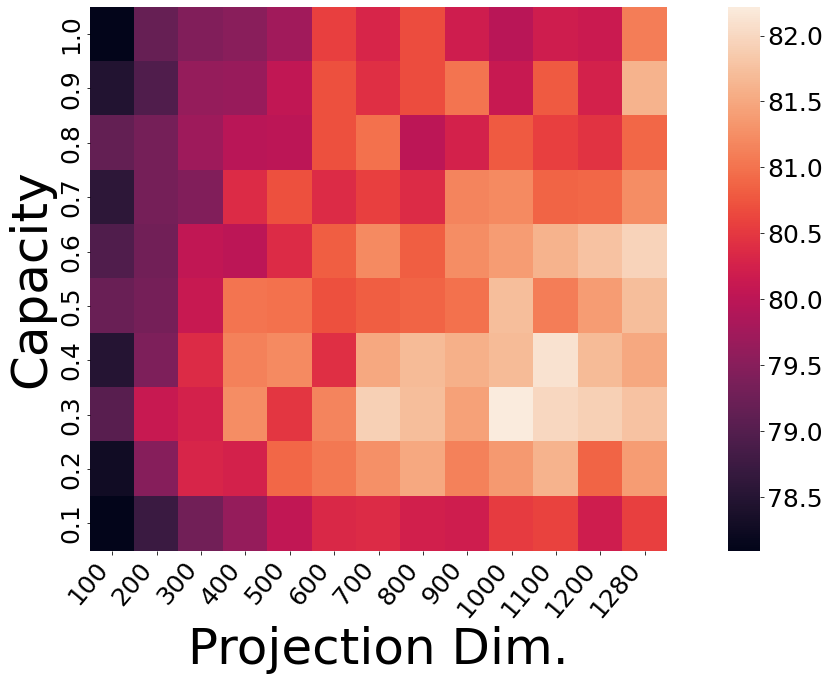}
        \subcaption[]{\small CIFAR-10.}
        \label{f:heatmap_cifar10}
    \end{subfigure}
    \hfill
    \begin{subfigure}[b]{0.48\textwidth}
        \centering
        \includegraphics[width=\textwidth, height=0.8\textwidth]{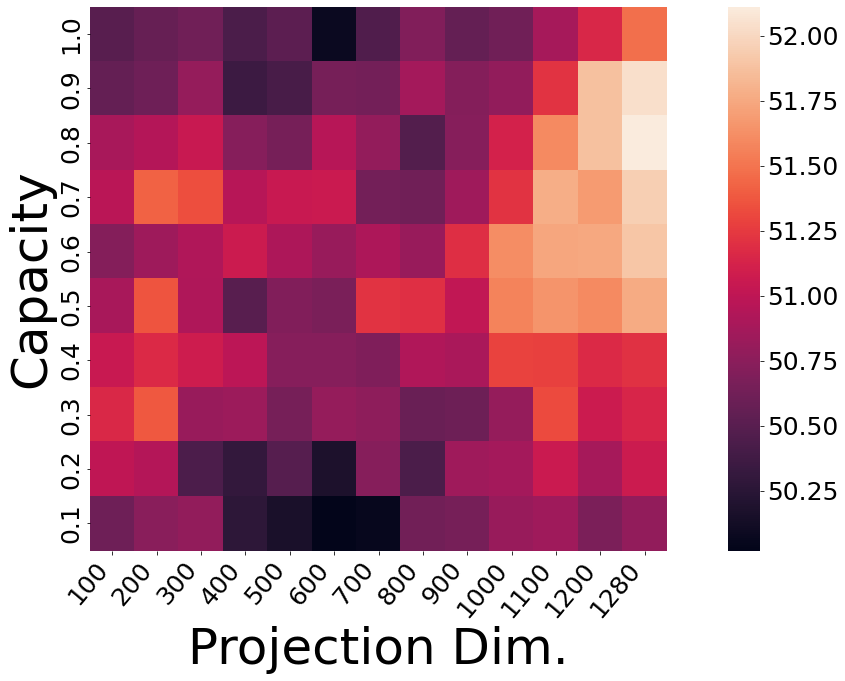}
        \subcaption[]{\small CIFAR-100.}
        \label{f:heatmap_cifar100}
    \end{subfigure}
    \hfill
    \caption[]
    {\small Test accuracy when the kNN mechanism is implemented through \texttt{ProtoNN} for different values of projection dimension and  number of prototypes (expressed as a fraction of the local dataset). CIFAR-10 (left) and CIFAR-100 (right) datasets.}
    \label{f:heatmap}
\end{figure*}

\begin{figure*}[!t]
    \centering
    \begin{subfigure}[b]{0.48\textwidth}  
        \centering 
        \includegraphics[width=\textwidth, height=0.8\textwidth]{Figures/CIFAR10/base_model_effect.pdf}
        \subcaption[]{\small CIFAR-10.}
    \end{subfigure}
    \hfill
    \begin{subfigure}[b]{0.48\textwidth}
        \centering
        \includegraphics[width=\textwidth, height=0.8\textwidth]{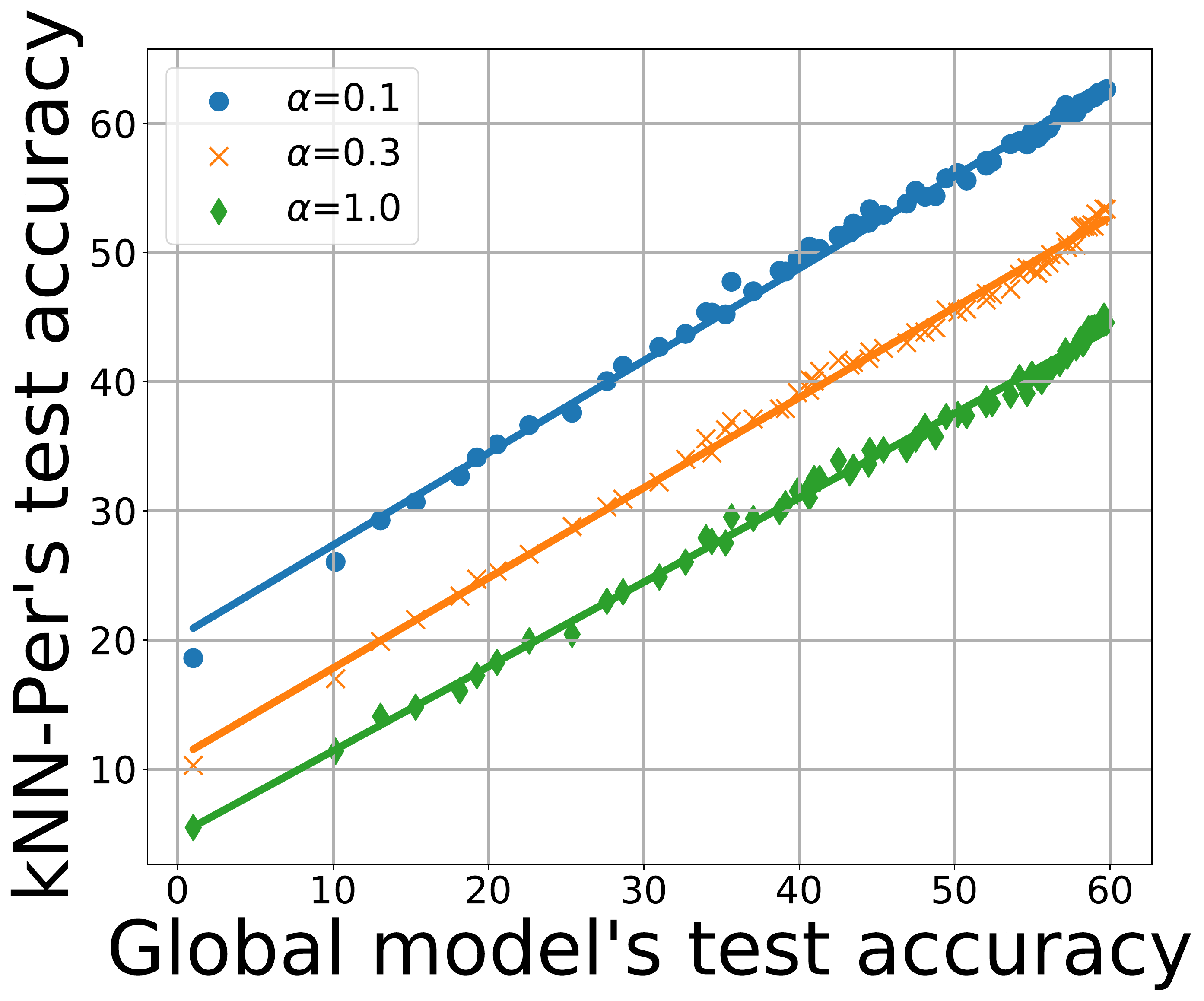}
        \subcaption[]{\small CIFAR-100.}
    \end{subfigure}
    \hfill
    \caption[]
    {\small Effect of the global model quality on the test accuracy of \texttt{kNN-Per} with $\lambda_m=1$ for each $m \in [M]$.}
    \label{f:base_model_effect}
\end{figure*}

\begin{figure*}[!t]
    \centering
    \begin{subfigure}[b]{0.48\textwidth}  
        \centering 
        \includegraphics[width=\textwidth, height=0.8\textwidth]{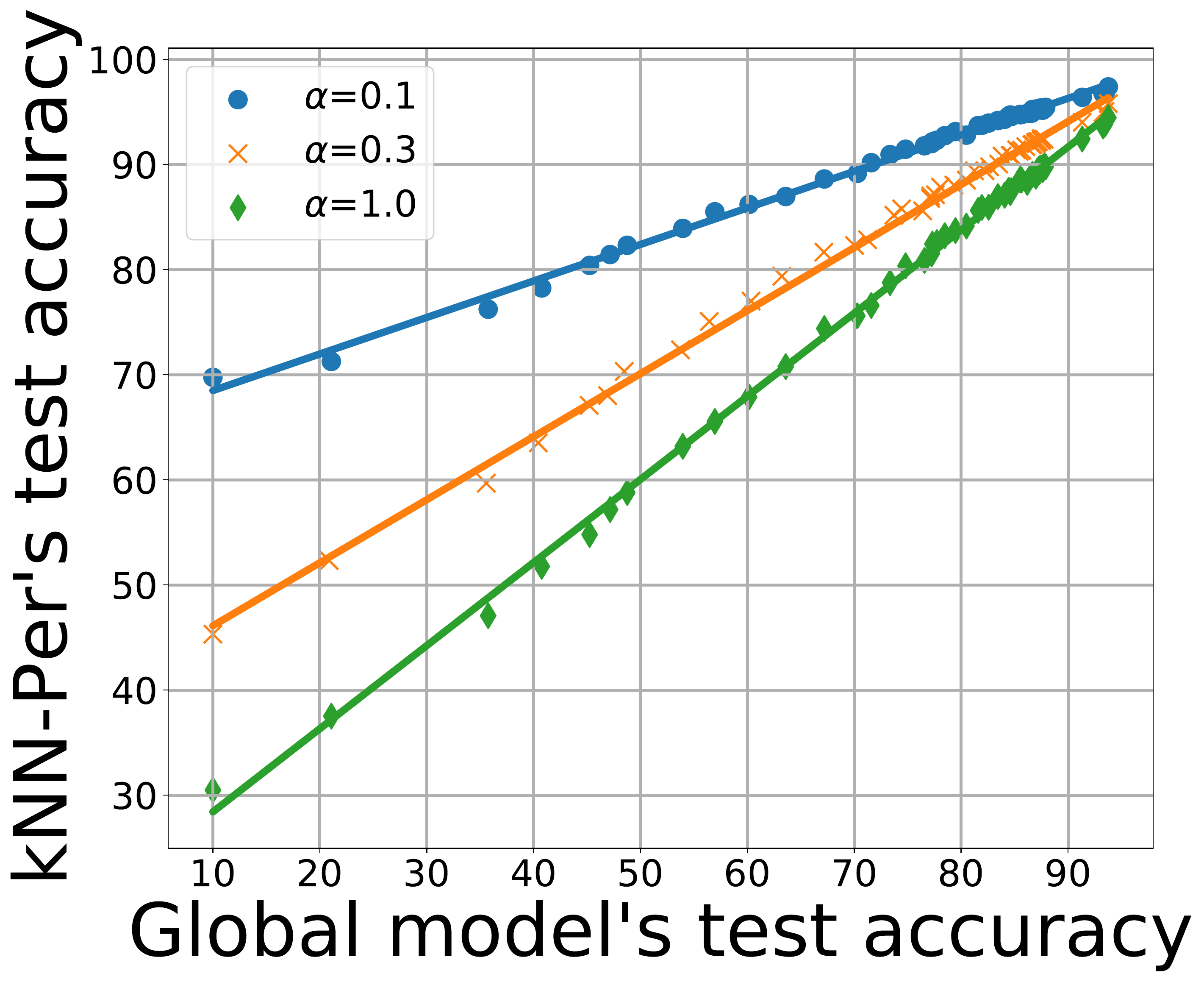}
        \subcaption[]{\small CIFAR-10.}
    \end{subfigure}
    \hfill
    \begin{subfigure}[b]{0.48\textwidth}
        \centering
        \includegraphics[width=\textwidth, height=0.8\textwidth]{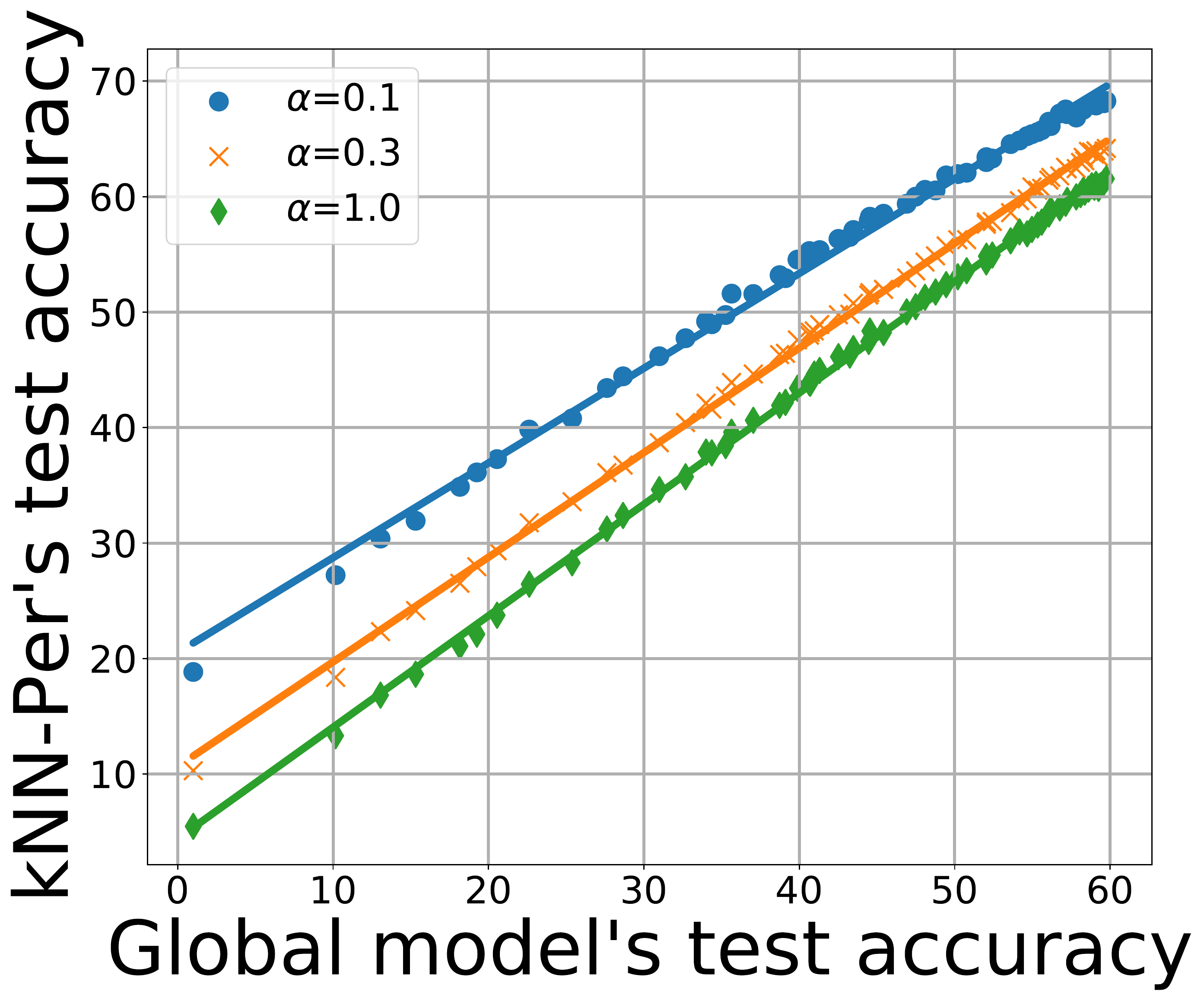}
        \subcaption[]{\small CIFAR-100.}
    \end{subfigure}
    \hfill
    \caption[]
    {\small Effect of the global model quality on the test accuracy of \texttt{kNN-Per} with $\lambda_m$ tuned per client.}
    \label{f:base_model_effect_lambda_opt}
\end{figure*}

\begin{figure*}[t]
    \centering
    \begin{subfigure}[b]{0.48\textwidth}  
        \centering 
        \includegraphics[width=\textwidth, height=0.8\textwidth]{Figures/CIFAR10/stream.pdf}
        \subcaption[]{\small CIFAR-10.}
        \label{f:stream_cifar10}
    \end{subfigure}
    \hfill
    \begin{subfigure}[b]{0.48\textwidth}
        \centering
        \includegraphics[width=\textwidth, height=0.8\textwidth]{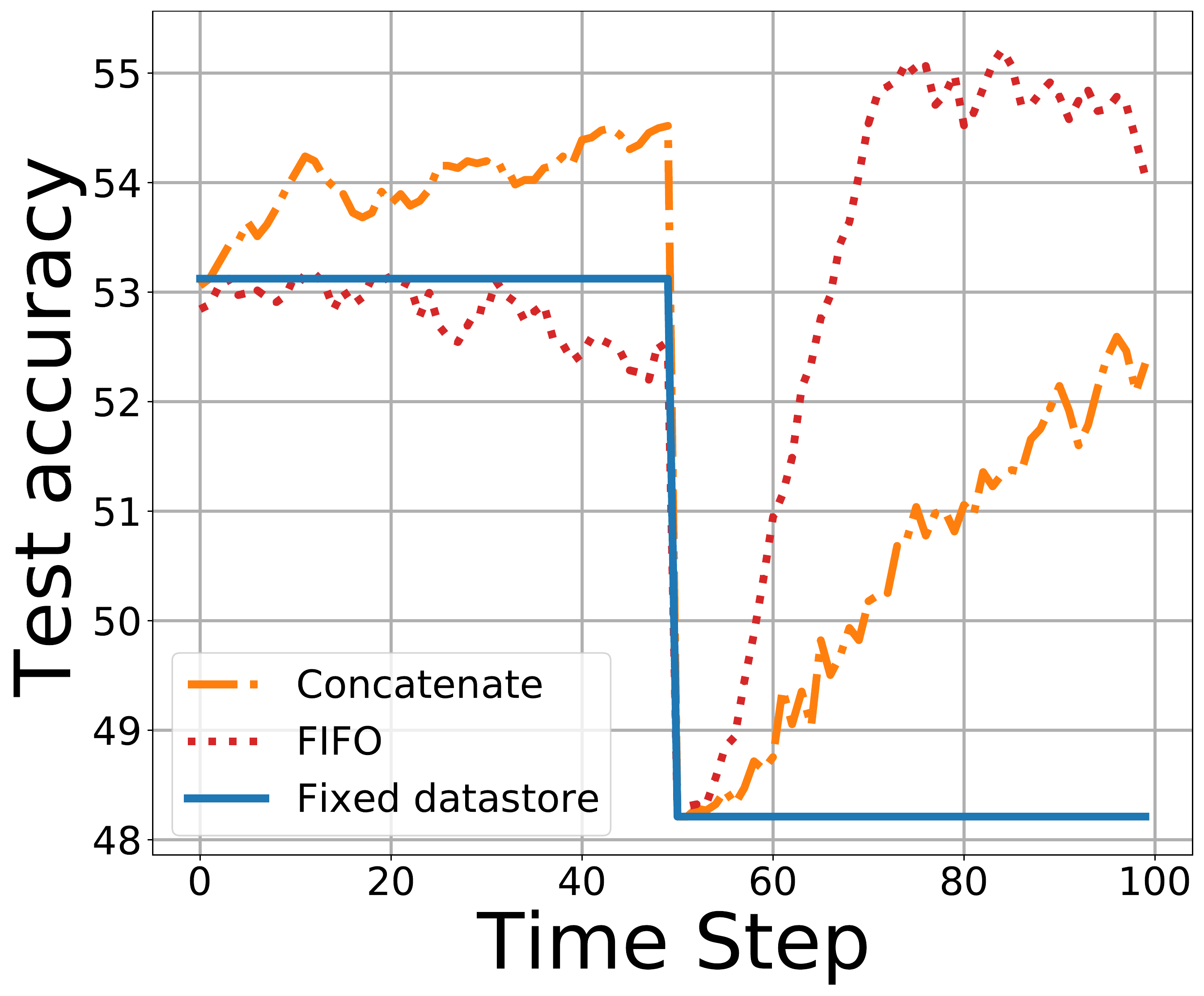}
        \subcaption[]{\small CIFAR-100.}
        \label{f:stream_cifar100}
    \end{subfigure}
    \hfill
    \caption[]
    {\small Test accuracy when a distribution shift happens at time step $t_{0}=50$ for different datastore management strategies.}
    \label{f:stream}
\end{figure*}

\end{document}